%% file: learning-model.tex
\documentclass[
	final,
	paper=a4,
	pagesize=auto,
	fontsize=10pt,
	version=last,
]{scrartcl}

\usepackage[utf8]{inputenc}
\usepackage[
	english,
]{babel}

\input{preamble}
\hypersetup{
	pdfauthor={Martin Genzel and Gitta Kutyniok},
	pdftitle={A Mathematical Framework for Feature Selection from Real-World Data with Non-Linear Observations},
	pdfsubject={Article},
	pdfcreator={PDF-LaTeX},
}
\input{macros/environments}
\input{macros/math}
\input{macros/content-macros}

\graphicspath{{images/}}

\begin{document}


\pagestyle{scrheadings}

\noindent{\Large\raggedright\bfseries \vphantom{g}A Mathematical Framework for Feature Selection from \\ Real-World Data with Non-Linear Observations}

\vspace{1\baselineskip}
\begin{addmargin}[6em]{0em}
\noindent{\normalsize\bfseries\larger{Martin Genzel and Gitta Kutyniok}}

\noindent Technische Universit\"at Berlin, Department of Mathematics \\
Straße des 17. Juni 136, 10623 Berlin, Germany \\
\noindent E-Mails: \href{mailto:genzel@math.tu-berlin.de}{\texttt{[genzel,kutyniok]@math.tu-berlin.de}}

\vspace{1\baselineskip}
{\smaller
\noindent\textbf{Abstract.}
In this paper, we study the challenge of \emph{feature selection} based on a relatively small collection of sample pairs $\{(\data_i, y_i)\}_{1 \leq i \leq m}$. 
The \emph{observations} $y_i \in \R$ are thereby supposed to follow a noisy single-index model, depending on a certain set of \emph{signal variables}.
A major difficulty is that these variables usually cannot be observed directly, but rather arise as \emph{hidden factors} in the actual \emph{data vectors} $\data_i \in \R^d$ (\emph{feature variables}).
We will prove that a successful variable selection is still possible in this setup, even when the applied estimator does not have any knowledge of the underlying model parameters and only takes the ``raw'' samples $\{(\data_i, y_i)\}_{1 \leq i \leq m}$ as input.  
The model assumptions of our results will be fairly general, allowing for \emph{non-linear observations}, arbitrary \emph{convex signal structures} as well as \emph{strictly convex loss functions}.
This is particularly appealing for practical purposes, since in many applications, already standard methods, e.g., the Lasso or logistic regression, yield surprisingly good outcomes.
Apart from a general discussion of the practical scope of our theoretical findings, we will also derive a rigorous guarantee for a specific real-world problem, namely sparse feature extraction from \emph{(proteomics-based) mass spectrometry data}.

\noindent\textbf{Key words.}
Convex optimization, dictionary representations, feature selection, Gaussian mean width, high-dimensional data, mass spectrometry data, model uncertainty, non-linear observations, 
sparsity, structured loss minimization, variable selection.
}
\end{addmargin}
\newcommand{\shortauthor}{M. Genzel and G. Kutyniok}

\thispagestyle{plain}

\input{content}


\section*{Acknowledgements}
This research is supported by the Einstein Center for Mathematics Berlin (ECMath), project grant CH2.
G. Kutyniok acknowledges partial support by the Einstein Foundation Berlin, the Einstein Center for Mathematics Berlin (ECMath), the 
European Commission-Project DEDALE (contract no. 665044) within the H2020 Framework Program, DFG Grant KU 1446/18, DFG-SPP 
1798 Grants KU 1446/21 and KU 1446/23, the DFG Collaborative Research Center TRR 109 Discretization in Geometry and Dynamics, 
and by the DFG Research Center Matheon Mathematics for Key Technologies in Berlin.

\renewcommand*{\bibfont}{\small}
\bibliographystyle{abbrv}
\bibliography{learning-model}

\end{document}

%% file: preamble.tex
\KOMAoptions{%
   DIV=10,
}%
\KOMAoptions{%
   numbers=enddot
}%
\KOMAoptions{
   twoside=semi,  
   twocolumn=false, 
   headinclude=false,%
   footinclude=false,%
   mpinclude=false,%
   headlines=2.1,%
   headsepline=true,%
   footsepline=false,%
   cleardoublepage=empty 
}%
\KOMAoptions{%
   parskip=half-  
}%
\KOMAoptions{%
   toc=bib,
}%
\KOMAoptions{%
}%
\KOMAoptions{%
   bibliography=openstyle,%
}%
\KOMAoptions{%
}%
\KOMAoptions{%
}%
\KOMAoptions{%
}%

\makeatletter
\newcommand{\LoadPackagesNow}{}
\newcommand{\LoadPackageLater}[2][]{%
   \g@addto@macro{\LoadPackagesNow}{%
      \usepackage[#1]{#2}%
   }%
}
\makeatother

\usepackage{xspace}
\usepackage{xargs}
\usepackage{ifthen}
\usepackage{etoolbox}

\usepackage[
	margin=3cm,
	]{geometry}
\usepackage{framed}
\usepackage{mdframed}
\usepackage{pbox}
\usepackage{enumitem}
\usepackage{array}
\usepackage{tabu}
\usepackage{multirow}
\usepackage{longtable}
\usepackage{hhline}
\usepackage{afterpage}
\usepackage{pdflscape}
\usepackage{rotating}
\usepackage{booktabs}
\usepackage{setspace}
\usepackage[
   bottom,      
   stable,      
   perpage,     
   ragged,      
   multiple,    
   flushmargin,
   hang,
]{footmisc}

\usepackage{mathpazo}

\usepackage[%
   headsepline,
   automark,
   komastyle,
]{scrpage2}

\clearscrheadings
\clearscrplain
\lehead{\pagemark}
\rohead{\pagemark}
\rehead{\headmark}
\lohead{\shortauthor}
\automark[section]{section} 

\usepackage{relsize}
\usepackage[normalem]{ulem}      
\usepackage{soul}		           
\usepackage{url}
\usepackage{lipsum}

\usepackage[table]{xcolor}
\usepackage[%
]{graphicx}

\usepackage{epstopdf}
\usepackage{wrapfig}
\usepackage{caption}
\usepackage{subcaption}
\usepackage[
   tbtags,    
   sumlimits,  
   nointlimits, 
   namelimits, 
   reqno,     
]{amsmath} %

\usepackage{amsfonts}
\usepackage{mathrsfs} 
\usepackage{dsfont}
\usepackage{amssymb}
\usepackage{units}
\LoadPackageLater{amsthm}
\usepackage[fixamsmath,disallowspaces]{mathtools}
\mathtoolsset{showonlyrefs}
\mathtoolsset{centercolon=true}
\usepackage{bm} 
\makeatletter
\g@addto@macro\bfseries{\boldmath}
\makeatother
\allowdisplaybreaks[4]
\numberwithin{equation}{section}

\usepackage[toc]{appendix}

\usepackage[%
	square,	
	comma,	
	numbers,	
	sort,		
	sort&compress,    
]{natbib}

\usepackage{csquotes}

\usepackage[textsize=tiny,english,colorinlistoftodos,
	disable,
	]{todonotes}

\definecolor{pdfurlcolor}{rgb}{0,0,0.6}
\definecolor{pdffilecolor}{rgb}{0.7,0,0}
\definecolor{pdflinkcolor}{rgb}{0,0,0.6}
\definecolor{pdfcitecolor}{rgb}{0,0,0.6}
\usepackage[
  colorlinks=true,         
  urlcolor=pdfurlcolor,    
  filecolor=pdffilecolor,  
  linkcolor=pdflinkcolor,  
  citecolor=pdfcitecolor,  %
  raiselinks=true,			 
  breaklinks,              
  verbose,
  hyperindex=true,         
  linktocpage=true,        
  hyperfootnotes=false,     
  bookmarks=true,          
  bookmarksopenlevel=1,    
  bookmarksopen=true,      
  bookmarksnumbered=true,  
  bookmarkstype=toc,       
  plainpages=false,        
  pageanchor=true,         
  pdfdisplaydoctitle=true, 
  pdfstartview=FitH,       
  pdfpagemode=UseOutlines, 
  pdfpagelabels=true,           
  pdfpagelayout=OneColumn, 
]{hyperref}

\LoadPackagesNow

\newcommand{\ifargdef}[3][{}]{\ifthenelse{\equal{#2}{}}{#1}{#3}}

\setkomafont{section}{\centering\normalfont\rmfamily\scshape}
\setkomafont{subsection}{\rmfamily\bfseries}
\setkomafont{paragraph}{\rmfamily\bfseries}
\setkomafont{pageheadfoot}{\smaller\scshape\rmfamily}

%% file: macros/environments.tex
\newenvironment{highlight}{\begin{addmargin}[1em]{1em}\itshape}{\end{addmargin}}
\newcommand{\highlighted}[1]{\emph{#1}}

\newenvironment{properties}[2][2em]
{\begin{enumerate}[label={\textsc{(#2\arabic*)}},leftmargin=#1]}
{\end{enumerate}} 

\newenvironment{listing}
{\begin{itemize}[itemindent=0em,leftmargin=1.2em]}
{\end{itemize}}

\newenvironment{rmklist}
{\begin{enumerate}[label={(\arabic*)},itemindent=2em,leftmargin=0em]}
{\end{enumerate}}

\newcommand{\qeddiamond}{\hfill$\Diamond$}

\providecommand{\qedrmkhere}{\hfill\qeddiamond}

\newtheoremstyle{claim}
	{\topsep}{\topsep}%
	{\itshape}
	{}
	{}
	{}
	{.5em}
	{{\bfseries\boldmath\thmname{#1} \thmnumber{#2}} \thmnote{(#3)}}

\newtheoremstyle{definition}
	{\topsep}{\topsep}%
	{}
	{}
	{}
	{}
	{.5em}
	{\textbf{\thmname{#1} \thmnumber{#2}} \thmnote{(#3)}}
	
\newtheoremstyle{algorithm}
	{\topsep}{\topsep}%
	{}
	{}
	{\bfseries\boldmath}
	{}
	{.5em}
	{\thmname{#1} \thmnumber{#2} \thmnote{(#3)}}

\theoremstyle{claim}
\newtheorem{theorem}{Theorem}[section]

\newtheorem{corollary}[theorem]{Corollary}

\theoremstyle{definition}
\newtheorem{definition}[theorem]{Definition}
\newtheorem{remark}[theorem]{Remark}
\newtheorem{example}[theorem]{Example}

%% file: macros/math.tex
\newcommand{\opleft}[1]{\mathopen{}\left#1}
\newcommand{\opright}[1]{\right#1\mathclose{}}
\newcommandx{\braces}[4]{%
\ifstrequal{#3}{normal}{#1#4#2}{%
\ifstrequal{#3}{auto}{\left#1#4\right#2}{%
\ifstrequal{#3}{opauto}{\opleft#1#4\opright#2}{%
#3#1#4#3#2}}}%
}
\newcommandx{\opannot}[3][3=\downarrow]{\stackrel{\mathclap{\substack{#1 \\ #3 \vspace{2pt}}}}{#2}}
\newcommandx{\lineannot}[3][3=\rightarrow]{\mathllap{\boxed{\text{\textsmaller{#1}}} #3} #2}
\newcommandx{\multilineannot}[4][4=\rightarrow]{\mathllap{\boxed{\parbox{#1}{\RaggedRight\textsmaller{#2}}} #4} #3}
 
\newcommand{\R}{\mathbb{R}} 
\newcommand{\eps}{\varepsilon} 

\newcommand{\suchthat}[1][normal]{\ifstrequal{#1}{normal}{\mid}{#1|}} 

\newcommand{\cardinality}{\#} 
\newcommand{\intersec}{\cap} 

\newcommandx{\intvcl}[3][1=normal]{\braces{[}{]}{#1}{#2, #3}} 
\newcommandx{\intvop}[3][1=normal]{\braces{(}{)}{#1}{#2, #3}} 
\newcommandx{\intvclop}[3][1=normal]{\braces{[}{)}{#1}{#2, #3}} 
\newcommandx{\intvopcl}[3][1=normal]{\braces{(}{]}{#1}{#2, #3}} 
\DeclareMathOperator*{\argmin}{argmin} 
\DeclareMathOperator{\sign}{sign}
\newcommandx{\abs}[2][1=normal]{\braces{\lvert}{\rvert}{#1}{#2}} 
\newcommandx{\ceil}[2][1=normal]{\braces{\lceil}{\rceil}{#1}{#2}} 
\newcommandx{\floor}[2][1=normal]{\braces{\lfloor}{\rfloor}{#1}{#2}} 
\newcommandx{\round}[2][1=normal]{\braces{[}{]}{#1}{#2}} 
\newcommandx{\der}[1]{D^{#1}} 
\newcommandx{\gradient}{\nabla} 
\newcommandx{\partder}[4][1={},4={}]{\frac{\partial^{#4} #2}{\partial #3^{#4}}\ifargdef{#1}{\Big|_{#1}}} 
\newcommandx{\integ}[4][1={},2={}]{\int_{#1}^{#2} #3 \, #4} 
\newcommandx{\asympffaster}[2][1=normal]{o\braces{(}{)}{#1}{#2}} 
\newcommandx{\asympfaster}[2][1=normal]{O\braces{(}{)}{#1}{#2}} 
\newcommandx{\asympeq}[2][1=normal]{\Theta\braces{(}{)}{#1}{#2}} 
\newcommandx{\asympsslower}[2][1=normal]{\omega\braces{(}{)}{#1}{#2}} 
\newcommandx{\asympslower}[2][1=normal]{\Omega\braces{(}{)}{#1}{#2}} 

\newcommand{\matr}[1]{\begin{bmatrix} #1 \end{bmatrix}} 
\newcommand{\smallmatr}[1]{\left[\begin{smallmatrix} #1 \end{smallmatrix}\right]} 
\newcommandx{\norm}[2][1=normal]{\braces{\|}{\|}{#1}{#2}} 
\renewcommandx{\sp}[3][1=normal]{\braces{\langle}{\rangle}{#1}{#2, #3}} 
\newcommandx{\End}[2][2={}]{\mathcal{L}\opleft( #1 \ifargdef{#2}{, #2} \opright)} 
\newcommand{\T}{\mathsf{T}} 
\renewcommand{\vec}[1]{\boldsymbol{#1}} 

\newcommandx{\measure}[2][1=normal]{\operatorname{vol}\braces{(}{)}{#1}{#2}} 
\DeclareMathOperator{\supp}{supp} 
\newcommandx{\Leb}[3][1={},3=normal]{L^{#2}\ifargdef{#1}{\braces{(}{)}{#3}{#1}}{}} 
\newcommandx{\Lebnorm}[4][1=normal,3={2},4={}]{\norm[#1]{#2}_{#3}} 
\renewcommandx{\l}[3][1={},3=normal]{\ell^{#2}\ifargdef{#1}{\braces{(}{)}{#3}{#1}}} 
\newcommandx{\lnorm}[4][1=normal,3={2},4={}]{\norm[#1]{#2}_{#3}} 
\newcommandx{\Smooth}[4][1={},3={},4=normal]{C_{#3}^{#2}\ifargdef{#1}{\braces{(}{)}{#4}{#1}}} 
\newcommandx{\Schwartz}[2][1={},2=normal]{\mathscr{S}\ifargdef{#1}{\braces{(}{)}{#2}{#1}}} 
\newcommandx{\Schwartzpoly}[2][1=normal]{\braces{\langle}{\rangle}{#1}{\abs[#1]{#2}} } 
\newcommandx{\Tempdistr}[2][1={},2=normal]{\mathscr{S}'\ifargdef{#1}{\braces{(}{)}{#2}{#1}}} 
\newcommandx{\distrinp}[3][1=normal]{\braces{\langle}{\rangle}{#1}{#2, #3}} 
\newcommandx{\ft}[3][1=default,2=auto]{
\ifstrequal{#1}{default}{\widehat{#3}}{
\ifstrequal{#1}{long}{{\braces{(}{)}{#2}{#3}}^{\wedge}}{}}} 
\newcommandx{\ift}[3][1=default,2=auto]{
\ifstrequal{#1}{default}{\check{#3}}{
\ifstrequal{#1}{long}{{\braces{(}{)}{#2}{#3}}^{\vee}}{}}} 


%% file: macros/content-macros.tex
\newcommand{\define}[1]{\emph{#1}}

\renewcommand{\v}{\vec{v}}
\newcommand{\meas}{\vec{u}}
\newcommand{\y}{\vec{y}}
\newcommand{\yrnd}{Y}
\newcommand{\Y}{\mathsf{Y}}
\newcommand{\fobs}{f}
\newcommand{\Fobs}{\solu{F}}
\newcommand{\scalfac}{\mu}
\newcommand{\modeldev}{\rho}
\newcommand{\modeldeveta}{\eta}
\newcommand{\modeldevconst}{C_{\modeldev,\modeldeveta}}

\newcommand{\advdev}{\eps}
\newcommand{\x}{\vec{x}}
\newcommand{\data}{\vec{x}}
\newcommand{\datarnd}{\vec{X}}
\newcommand{\lat}{\vec{s}}
\newcommand{\latrnd}{\vec{S}}
\newcommand{\latnoise}{\vec{n}}
\newcommand{\sig}{\vec{z}}
\newcommand{\trusig}{\tru{\sig}}
\newcommand{\trusigmu}{\scalfac\tru{\sig}}
\newcommand{\fv}{\vec{\beta}}

\newcommand{\atom}{\vec{a}}
\newcommand{\atoms}{\vec{A}}
\newcommand{\atomnoise}{\vec{b}}
\newcommand{\atomsnoise}{\vec{B}}
\newcommand{\dict}{\vec{D}}
\newcommand{\dictatom}{\vec{D}}
\newcommand{\dictenergy}{D_{\max}}
\newcommand{\dictnoise}{\vec{N}}
\newcommand{\noiseatom}{\vec{N}}
\newcommand{\snrscal}{\lambda}
\newcommand{\scalsig}{\tau}
\newcommand{\tru}[1]{{#1}_0}
\newcommand{\solu}[1]{\hat{#1}}
\newcommand{\std}[1]{\bar{#1}}
\newcommand{\extd}[1]{\tilde{#1}}
\newcommand{\sset}{K}
\newcommand{\loss}{\mathcal{L}}
\newcommand{\losssq}{\mathcal{L}^{\text{sq}}}
\newcommand{\lossemp}[1][{}]{\bar{\mathcal{L}}_{#1}}

\newcommand{\Idm}[1]{\vec{I}_{#1}}
\newcommand{\vnull}{\vec{0}}

\newcommand{\diag}[1]{\operatorname{diag}(#1)}
\newcommand{\vunit}{\vec{e}}
\newcommandx{\prob}[2][1={},2=normal]{\mathbb{P}\ifargdef{#1}{\braces{[}{]}{#2}{#1}}}
\newcommandx{\mean}[2][1={},2=normal]{\mathbb{E}\ifargdef{#1}{\braces{[}{]}{#2}{#1}}}
\newcommandx{\var}[2][1={},2=normal]{\mathbb{V}\ifargdef{#1}{\braces{[}{]}{#2}{#1}}}

\newcommand{\distributed}{\sim}

\newcommand{\Normdistr}[2]{\mathcal{N}(#1, #2)}
\newcommand{\gaussian}{\vec{g}}
\newcommand{\gaussianuniv}{g}
\newcommand{\stddev}{\sigma}
\newcommand{\Covmatr}{\vec{\Sigma}}
\newcommand{\SNR}{\operatorname{SNR}}

\newcommandx{\ball}[2][1={},2={}]{B_{#1}^{#2}}
\DeclareMathOperator{\convhull}{conv}
\renewcommand{\S}{S}
\newcommand{\meanwidth}[2][{}]{w_{#1}(#2)}
\newcommand{\effdim}[2][{}]{d_{#1}(#2)}

\newcommand{\tsqrt}[1]{{\scriptstyle\sqrt{#1}}}

%% file: content.tex
\section{Introduction}
\label{sec:intro}

\subsection{Motivation: Feature Selection from Proteomics-Based Data}
\label{subsec:intro:motivation}

Let us start with a classical problem situation from learning theory. Suppose we are given a collection of \define{samples} $(\lat_1, y_1), \dots, (\lat_m, y_m) \in 
\R^p \times \{-1,+1\}$ which are independently drawn from a random pair $(\latrnd, \yrnd)$ with unknown joint probability distribution on $\R^p \times \{-1,+1\}$. Here, the random vector $\latrnd$ typically models a set of \define{signal variables} (or \define{data variables}), whereas the binary \define{label} $\yrnd$ assigns this data to a certain \define{class-of-interest}, which is either $-1$ or $+1$ in our case. A major challenge of \define{supervised machine learning} is then to find an accurate 
\define{predictor} $\Fobs\colon \R^p \to \{-1,+1\}$ of this classification procedure, such that $\solu\yrnd := \Fobs(\latrnd)$ coincides with the true variable $\yrnd$, at least ``with high probability.'' In a very simple example scenario, we may assume that the observed labels can be described by a \define{linear classification model} of the form\footnote{This is the same as assuming that $\yrnd = \sign(\sp{\latrnd}{\trusig})$ because each sample $(\lat_i, y_i)$ can be seen as an independent realization of $(\latrnd, \yrnd)$. But in this paper, we shall prefer the ``sample notation'' of \eqref{eq:intro:binarymodel}, which seems to be more natural and convenient for our purposes.}
\begin{equation}\label{eq:intro:binarymodel}
	y_i = \sign(\sp{\lat_i}{\trusig}), \quad i = 1, \dots, m,
\end{equation}
where $\trusig \in \R^p$ is an unknown \define{signal vector} (or \define{parameter vector}).
The ultimate goal would be now to \emph{learn} an estimator $\solu\sig \in \R^p$ of $\trusig$, by merely using a \emph{small} set of training pairs $\{(\lat_i, y_i)\}_{1 \leq i \leq m}$.
A good approximation of the true signal vector $\trusig$ would not only provide a reliable predictor $\Fobs(\latrnd) = \sign(\sp{\latrnd}{\solu\sig})$, but in fact, its non-zero entries $\supp(\solu\sig) = \cardinality\{ j \suchthat \solu{z}_j \neq 0 \}$ would even indicate which variables of the data $\latrnd$ are (strongly) correlated with the associated class $\yrnd$. Such a statement is of course much stronger than just correctly predicting the class label because in that way, we are able to understand the underlying observation process.
The main focus of this work will be precisely on this type of problem, which is usually referred to as the task of \define{feature selection}, \define{variable selection}, or \define{feature extraction} in statistical learning theory.

Before continuing with the general problem issue, let us illustrate the above setup by a specific real-world example: The medical research of the last decades has shown that the early diagnosis of tumor diseases, such as cancer, can be significantly improved by extracting new \define{bio\-markers} from \define{proteomics data}. In this context, each sample pair $(\lat_i, y_i) \in \R^p \times \{-1,+1\}$ corresponds to an individual proband of a clinical study.
The class label $y_i$ simply specifies whether the \mbox{$i$-th} test person suffers from a certain disease ($y_i = -1$) or not ($y_i = +1$), whereas each single entry (variable) of the data $\lat_i = (s_{i,1}, \dots, s_{i,p})$ contains the \emph{molecular concentration} of a particular protein structure in the human body.
In this sense, $\lat_i$ represents (a part of) the so-called \define{proteome}, which is the entire collection of an individual's proteins at a fixed point of time.
Assuming a linear classification model as in \eqref{eq:intro:binarymodel} would mean that the patient's health status can be essentially determined by the presence or absence of some of those protein structures.
The signal vector $\trusig$ now plays the role of a \define{disease fingerprint} because its non-zero entries precisely indicate those proteins which seem to be relevant to the examined disease. Interestingly, various empirical studies have shown that, oftentimes, already a very small set of highly discriminative molecules is characteristic for a certain disease (see \cite{conrad2015spa} and the references therein), which implies that $\trusig$ might be relatively \define{sparse}.\footnote{A vector is said to be \emph{sparse} if most of its entries are equal to zero.}
The molecular concentrations of these few proteins finally form candidates for \define{biomarkers}, that are, reliable indicators for a potential affection of the body.
This prototype application will recurrently serve as an illustration of our framework. In order to keep the examples as simple as possible, we shall omit several nonessential biological details in this work. 
The interested reader is referred to \cite[Chap.~2]{genzel2015master} and \cite{conrad2015spa} for a more extensive discussion of the biological and clinical background.

Returning to our initial challenge of variable selection, we may ask under which conditions an accurate reconstruction of $\trusig$ from $\{(\lat_i, y_i)\}_{1 \leq i \leq m}$ is possible.
For instance, when assuming a Gaussian distribution of the data, $\lat_i \distributed \Normdistr{\vnull}{\Covmatr}$ with $\Covmatr \in \R^{p\times p}$ positive definite, this task becomes practically feasible (even in a much more general setting). Efficient algorithms as well as rigorous recovery guarantees are indeed available in this case; see \cite{genzel2016estimation,plan2013robust,plan2015lasso,ai2014onebitsubgauss,jacques2013onebit} for example.

Unfortunately, the situation is more complicated in most practical applications.
The signal variables $\lat_i = (s_{i,1}, \dots, s_{i,p})$ usually cannot be observed directly but merely in terms of a certain \define{data representation}. A typical model for a \define{real-world data set} may look as follows:
\begin{equation}\label{eq:intro:datamodel}
	\data_i = \sum_{k = 1}^p s_{i,k} \atom_k + \std\latnoise_i \in \R^d, \quad i = 1, \dots, m,
\end{equation}
where $\atom_1, \dots, \atom_p \in \R^d$ are fixed \define{feature atoms} (or \define{patterns}) and $\std\latnoise_i \in \R^d$ generates additive random \define{noise} in each entry.
The (hidden) variables $s_{i,1}, \dots, s_{i,p}$ are rather encoded as (random) scalar factors of a linear combination now, building up the data vector $\data_i$; this is why the $s_{i,1}, \dots, s_{i,p}$ are sometimes also called \define{signal factors}.

In the setup of proteomics, the model of \eqref{eq:intro:datamodel} could, for example, describe \define{mass spectrometry data} (\define{MS~data}), which is a widely-used acquisition method to detect the concentration of protein structures within clinical samples, e.g., blood or urine.
A typical mass spectrum is shown in Figure~\ref{fig:intro:motivation:msdata}.
We may assume that each single feature atom $\atom_k \in \R^d$ corresponds to a particular peak, determining its \define{position} and \define{shape}, whereas its \define{height} is specified by the scalar factor $s_{i,k}$, which varies from sample to sample. 
Thus, the information-of-interest is not directly available anymore, but only represented in terms of the (relative) peak heights.
The technical and physical details of MS~data will be further discussed in Subsection~\ref{subsec:applications:msdata}, including a precise definition of the data model.
\begin{figure}
	\centering
	\includegraphics[width=1\textwidth]{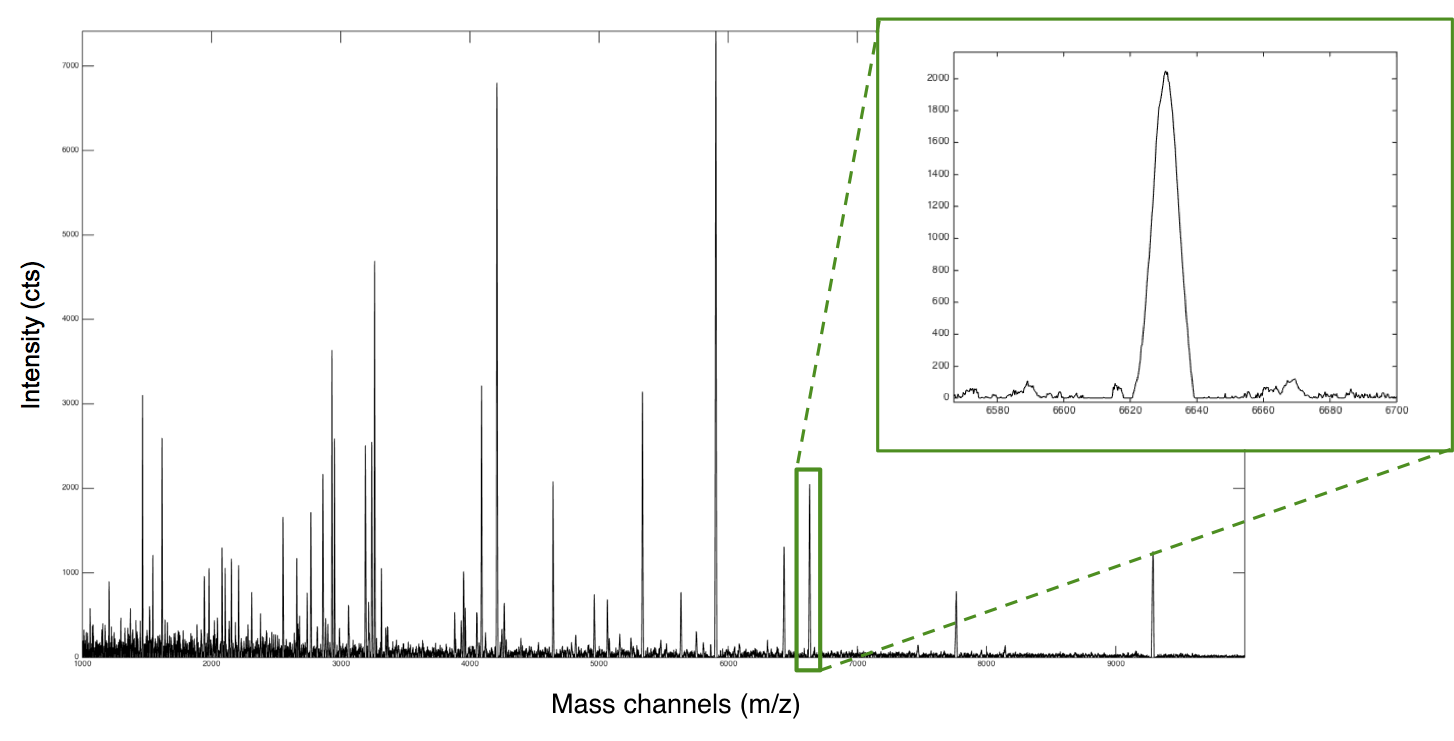}
	\caption{A typical example of a (finite-resolution) mass spectrum.
	The horizontal axis represents the indices $j = 1, \dots, d$ of the associated data vector $\data_i = (x_{i,1}, \dots, x_{i,d})$, plotted against its entries on the vertical axis (in the above spectrum, we have $d = 42\ 390$).
	Each of the characteristic \emph{Gaussian-shaped peaks} can be identified with a specific type of protein, where its maximal intensity is proportional to the molecular concentration within the examined sample.
	}
	\label{fig:intro:motivation:msdata}
\end{figure}

As a general conclusion, we are often obliged to deal with data pairs $\{(\data_i, y_i)\}_{1 \leq i \leq m}$---for example, generated by \eqref{eq:intro:datamodel} and \eqref{eq:intro:binarymodel}, respectively---but our primal goal is still to estimate the signal vector $\trusig$ from them. 
The main difficulty is that in many practical situations neither the feature atoms $\atom_1,\dots, \atom_p$ nor the signal factors $\lat_1, \dots, \lat_m$ are explicitly known.
In particular, there is no straightforward way to evaluate the binary observation scheme \eqref{eq:intro:binarymodel}.
One common approach is to perform a \define{factor analysis} for the data model \eqref{eq:intro:datamodel}, i.e., using the available sample collection $\data_1, \dots, \data_m$ to approximate both $\atom_1,\dots, \atom_p$ and $\lat_1, \dots, \lat_m$. But as long as the sample count $m$ is small and noise is present, such a procedure could become very unstable, and moreover, the desired factorization might be highly non-unique.

To circumvent these drawbacks, let us follow a more basic strategy that extracts features directly from the data.
For that matter, one could simply try to mimic the behavior of the classification model \eqref{eq:intro:binarymodel} by learning a \define{feature vector} $\solu\fv \in \R^d$ such that\footnote{In the literature, \emph{feature vectors} are often referred to as the outputs of a \emph{feature map}. But in this work, a feature vector $\fv \in \R^d$ is rather understood to be a \emph{vector that extracts features from the data} by a linear projection $\sp{\data_i}{\fv}$.}
\begin{equation}\label{eq:intro:binarymodeldd}
	y_i = \sign(\sp{\data_i}{\solu\fv})
\end{equation}
holds at least for ``many'' $i = 1, \dots, m$.\footnote{The word ``many'' essentially means that the equality of \eqref{eq:intro:binarymodeldd} holds for a large fraction of samples. This will be made more precise later on in Section~\ref{sec:results}.}
Similarly to the initial problem of variable selection, the non-zero entries of $\solu\fv$ may determine several features of the data that are relevant to an accurate prediction of the label.
For the case of (proteomics-based) MS~data, we would therefore try to choose the support of $\solu\fv$ as small as possible, but in such a way that its entries are closely located to those peaks which are strongly correlated with the examined disease.
Finally, a physician could use his/her expertise to identify these (few) peaks with the associated proteins ``by hand.'' From a practical perspective, this would essentially solve the original problem, although the molecular concentrations $\lat_i$ are still unknown.
Such a simple work-flow for (sparse) feature selection from proteomics data can lead in fact to state-of-the-art results for both simulated and real-world data sets; see \cite{conrad2015spa,genzel2015master} for some numerical experiments.
These works promote that many classical algorithms from machine learning, combined with some appropriate preprocessing, already provide a reliable feature vector $\solu\fv$. Perhaps, the most prominent example is the \define{Lasso}, originally introduced by Tibshirani \cite{tibshirani1996lasso}:
\begin{equation}
	\solu\fv = \argmin_{\fv \in \R^d} \tfrac{1}{2m} \sum_{i = 1}^m (\sp{\data_i}{\fv} - y_i)^2 \quad \text{subject to $\lnorm{\fv}[1] \leq R$,} \label{eq:intro:lasso}\tag{$P_{R}$}
\end{equation}
where the \define{regularization parameter} $R > 0$ controls the sparsity of the minimizer.

Inspired by the previous (rather empirically-based) observations, the major objective of this work is to develop an abstract framework for feature selection from real-world data, including provable guarantees for a fairly general class of problem instances. In this context, we particularly aim to address the following three questions:
\begin{properties}[3em]{Q}
\item\label{question:intro:theoretical}
	\highlighted{Which information about $\trusig$ does an estimated feature vector $\solu\fv$ already carry?
    Can we use $\solu\fv$ to define an accurate estimator $\solu\sig = \solu\sig(\solu\fv)$ of $\trusig$? If so, how do we measure the approximation error?}
\item\label{question:intro:practical}
	\highlighted{Regarding the issues of \ref{question:intro:theoretical}, to which extent can we aomit explicit knowledge of the data model \eqref{eq:intro:datamodel}? More precisely, in how far is the definition of $\solu\sig$ affected by the unknown feature atoms $\atom_1, \dots, \atom_p$?}
\item\label{question:intro:setup}
	\highlighted{In how far can we extend the above model setup? In particular: Can we replace the $\sign$-function in \eqref{eq:intro:binarymodel} by a general (unknown and noisy) non-linearity? 
	What other types of convex loss functions and structural constraints could be considered for the Lasso-estimator \eqref{eq:intro:lasso}?
	}
\end{properties}

It will turn out that, although the first two questions are obviously closely connected, their answers might be quite different. Indeed, while \ref{question:intro:theoretical} and \ref{question:intro:setup} are rather theoretically motivated and we will be able to give general solutions, a positive statement on \ref{question:intro:practical} will heavily depend on the intrinsic structure of the data model \eqref{eq:intro:datamodel}.

\subsection{Key Concepts: Signal Domain, Data Domain, and Optimal Representations}
\label{subsec:intro:keyidea}

Let us first focus on the problem issue of \ref{question:intro:theoretical}.
As already mentioned above, our scope is initially restricted to the \define{data domain} $\R^d$ where the samples $\data_1, \dots, \data_m \in \R^d$ belong.
But our actual task is associated with the \define{signal domain} $\R^p$, which contains the hidden signal variables $\lat_1, \dots, \lat_m$ and the desired parameter vector $\trusig$. Hence, the following question emerges:
\highlighted{How can we express the underlying observation procedure \eqref{eq:intro:binarymodel} in terms of the data domain?}
Inspired by the previous subsection, we shall simply assume that there exists \emph{some} feature vector $\tru\fv \in \R^d$ such that
\begin{equation}\label{eq:intro:trudatamodel}
	y_i = \sign(\sp{\lat_i}{\trusig}) = \sign(\sp{\data_i}{\tru\fv}) \quad \text{for ``many'' $i = 1, \dots, m$.}
\end{equation}
However, this choice of $\tru\fv$ is typically non-unique. The Gaussian-shaped peaks of MS~data, for example, have a certain spatial extent (see Figure~\ref{fig:intro:motivation:msdata}) so that the data vectors $\data_1, \dots, \data_m$ consist of several almost collinear entries. Due to this \emph{redundancy}, we may find another feature vector $\solu\fv \in \R^d$ whose support is disjoint from the one of $\tru\fv$ but satisfies \eqref{eq:intro:trudatamodel} as well.
It is therefore not very meaningful to compare feature vectors only in terms of their supports or their Euclidean distance.

In order to further investigate the properties of $\tru\fv$ and to derive an appropriate measure of distance for feature vectors,
we need to incorporate the factor model of \eqref{eq:intro:datamodel}. Collecting the atoms in a \define{feature matrix} $\atoms := \matr{\atom_1 \mid \dots \mid \atom_p} \in \R^{d \times p}$, we obtain a compact expression $\data_i = \atoms \lat_i + \std\latnoise_i$, $i = 1, \dots, m$, that leads to the following identity:
\begin{equation}\label{eq:intro:fromddtosd}
	\sign(\sp{\data_i}{\tru\fv}) = \sign(\sp{\atoms \lat_i}{\tru\fv} + \sp{\std\latnoise_i}{\tru\fv}) = \sign(\sp{\lat_i}{\atoms^\T\tru\fv} + \sp{\std\latnoise_i}{\tru\fv}).
\end{equation}
Supposed that the additive noise term $\sp{\std\latnoise_i}{\tru\fv}$ is relatively small compared to the signal term $\sp{\lat_i}{\atoms^\T\tru\fv}$---which essentially means that the \define{signal-to-noise ratio} (\define{SNR}) of the data is high---we may expect that
\begin{equation}\label{eq:intro:ddrepresentation}
	y_i = \sign(\sp{\lat_i}{\trusig}) = \sign(\sp{\lat_i}{\atoms^\T\tru\fv}) \quad \text{for ``many'' $i = 1, \dots, m$.}
\end{equation}
These simple calculations show us how the hidden classification rule can be (approximately) represented in terms of the feature vector $\tru\fv \in \R^d$. 
In particular, the matrix $\dict := \atoms^\T \in \R^{p \times d}$ thereby arises as a natural transform between the signal domain and data domain.
Note that $\dict$ just contains all feature atoms $\atom_1, \dots, \atom_p$ as rows, which should be primarily seen as \emph{fixed} and \emph{intrinsic} parameters of the data model \eqref{eq:intro:datamodel}.

If, in addition, $\dict$ is surjective,\footnote{This automatically implies that $p \leq d$. In fact, if the assumption of surjectivity is not satisfied, there is a certain evidence that the available data set might not be ``rich enough'' to describe the original problem of variable selection (associated with the signal domain).} we may even choose $\tru\fv \in \R^d$ such that $\trusig = \dict \tru\fv$ and \eqref{eq:intro:ddrepresentation} holds for \emph{all} samples $i = 1, \dots, m$.
In this context, it is quite common to regard $\dict$ as an overcomplete \define{feature dictionary} that allows us to \emph{represent} the signal-of-interest $\trusig$ by a \emph{coefficient vector} $\tru\fv$ from the data domain.
But nevertheless, the choice of $\tru\fv$ could be still highly non-unique. So one might ask for the following:
\highlighted{What is the ``best-possible'' representation of $\trusig$ by the dictionary $\dict$?}
On the one hand, we should make use of the fact that $\trusig$ often carries some additional (low-complexity) structure.
As an example, one could try to mimic the sparsity constraint of the Lasso \eqref{eq:intro:lasso} and assume that there exists $\tru\fv \in R \ball[1][d]$ with $R > 0$ small and $\trusig = \dict \tru\fv$;\footnote{Here, $\ball[1][d] = \{ \data \in \R^d \suchthat \lnorm{\data}[1] \leq 1 \}$ denotes the $\l{1}$-unit ball of $\R^d$.} this would actually mean that $\trusig$ can be \emph{sparsely} represented by $\dict$.
On the other hand, it is also essential to control the deviation of the noise term $\sp{\std\latnoise_i}{\tru\fv}$ in \eqref{eq:intro:fromddtosd}, so that the mismatch of our approximation in \eqref{eq:intro:trudatamodel} does not become too large. 
Hence, we may simply minimize its variance:
\begin{equation}\label{eq:intro:noisevariance}
	\tru\stddev^2 := \min_{\substack{\fv \in R \ball[1][d] \\ \trusig = \dict \fv}} \mean[\sp{\std\latnoise_i}{\fv}^2].
\end{equation}
This definition forms a key ingredient of our general framework in Section~\ref{sec:results}.
In fact, any minimizer $\tru\fv \in R \ball[1][d]$ of \eqref{eq:intro:noisevariance} will be referred to as an \define{optimal representation of $\trusig$ by $\dict$} (see Definition~\ref{def:results:dictionary:optimalrepr}).
We will particularly observe that the corresponding \define{noise variance} $\tru\stddev^2$ is immediately related to the SNR, which in turn plays a crucial role in the quality of our error estimates.

Let us finally return to the initial question of \ref{question:intro:theoretical}. Given an accurate classifier $\solu\fv \in \R^d$, the above discussion suggests to define an estimator of $\trusig$ via the dictionary transform $\dict$, that is, $\solu\sig := \dict \solu\fv$.
It will turn out that this idea, although surprisingly simple, is the key step to achieve a rigorous error bound on the Euclidean distance
\begin{equation}\label{eq:intro:errormeasure}
	\lnorm{\solu\sig - \snrscal \trusig} = \lnorm{\dict \solu\fv - \snrscal \dict \tru\fv},
\end{equation}
where $\snrscal$ is an adaptive \define{scaling parameter} that depends on the SNR as well as on the (non-)linear observation model. In particular, \eqref{eq:intro:errormeasure} should be regarded as an intrinsic and natural (semi-)metric to compare two feature vectors within the data domain.

\subsection{A First Glimpse of Our Main Result}
\label{subsec:intro:guarantee}

Before developing the ideas of the previous subsection towards a rigorous mathematical framework in Section~\ref{sec:results}, we would like to briefly illustrate our approach by a first theoretical guarantee for feature selection, which is a simplified and easy-to-read version of our main results Theorem~\ref{thm:results:selection} and Theorem~\ref{thm:results:selection:refined}.
In order to stress the key concepts and to avoid technicalities, let us consider the \emph{noiseless} case here: This means, we assume that the data pairs $\{(\data_i, y_i)\}_{1 \leq i \leq m} \in \R^d \times \{-1,+1\}$ are i.i.d. samples of the models \eqref{eq:intro:binarymodel} and \eqref{eq:intro:datamodel}, where $\lat_i \distributed \Normdistr{\vnull}{\Idm{p}}$ and $\std\latnoise \equiv \vnull$.
Moreover, let $\dictatom_1, \dots, \dictatom_d \in \R^p$ denote the columns of the dictionary $\dict \in \R^{p \times d}$, which are usually referred to as the \define{dictionary atoms}.\footnote{This should not be mixed up with the \emph{feature atoms} $\atom_1, \dots, \atom_p \in \R^d$, which form the rows of $\dict$ and are rather seen as the ``building blocks'' of the data.}
The ``maximal energy'' of the feature variables is then given by the largest column norm:
\begin{equation}
	\dictenergy := \max_{1 \leq j \leq d} \lnorm{\dictatom_j}.
\end{equation}
Using this model setup, we can state the following theorem with its proof being postponed to Subsection~\ref{subsec:proofs:sparseselection}:
\begin{theorem}\label{thm:intro:sparseselection}
	Let the above model assumptions hold true. Moreover, suppose that $\lnorm{\trusig} = 1$ and $\trusig \in R \dict\ball[1][d]$ for some $R > 0$, implying that we can find a representing feature vector $\tru\fv \in R\ball[1][d]$ with $\trusig = \dict \tru\fv$. Then there exist constants $C, C' > 0$ such that the following holds with high probability:
	
	If the number of samples obeys
	\begin{equation}\label{eq:intro:sparseselection:samplecount}
		m \geq C \cdot \dictenergy^2 \cdot R^2 \cdot \log(2d),
	\end{equation}
	then, setting $\solu\sig := \dict \solu\fv$ for any minimizer $\solu\fv$ of the Lasso \eqref{eq:intro:lasso}, we have
	\begin{equation}\label{eq:intro:sparseselection:bound}
		\lnorm{\solu\sig - \sqrt{\tfrac{2}{\pi}}\trusig} = \lnorm{\dict\solu\fv - \sqrt{\tfrac{2}{\pi}}\dict\tru\fv} \leq C' \bigg(\frac{\dictenergy^2 \cdot R^2 \cdot\log(2d)}{m}\bigg)^{1/4}.
	\end{equation}
\end{theorem}

In simple terms, the error bound \eqref{eq:intro:sparseselection:bound} (together with the condition of \eqref{eq:intro:sparseselection:samplecount}) yields a sufficient criterion for successful feature selection:\footnote{Some people strictly refer to \emph{variable selection} as the task of recovering the support of $\trusig$. In this work however, \emph{selecting variables} and \emph{approximating} $\trusig$ (in the Euclidean sense) are understood to be the same challenges. This is also due to the fact that we shall go beyond classical sparsity patterns later on.}
\highlighted{An accurate recovery of $\trusig$ is guaranteed as long as the sample count $m$ (greatly) exceeds $\dictenergy^2 \cdot R^2 \cdot \log(2d)$.}
Thus, at least in our simplified model setting, Theorem~\ref{thm:intro:sparseselection} gives a satisfactory answer to \ref{question:intro:theoretical}.
It is particularly important that the right-hand side of \eqref{eq:intro:sparseselection:bound} only \emph{logarithmically} depends on the dimension $d$ of the data, which can be extremely large in practice.
As a consequence, the above result indicates that feature selection is already possible if only a relatively few samples are available, which is also a typical constraint in realistic scenarios.

In contrast, Theorem~\ref{thm:intro:sparseselection} does not give a full solution to \ref{question:intro:practical}. The dictionary $\dict$ is unfortunately \emph{unknown} in many applications so that we are neither able to explicitly construct the estimator $\solu\sig$ nor to assess the quality of $\solu\fv$ in the sense of \eqref{eq:intro:sparseselection:bound}.
Therefore, it is a remarkable feature of Theorem~\ref{thm:intro:sparseselection} that variable selection by a standard algorithm is still (theoretically) feasible, although a factorization of the data $\data_i = \dict^\T \lat_i$ is completely missing.

Apart from that, the approximation quality of \eqref{eq:intro:sparseselection:bound} also depends on an appropriate bound for the product $\dictenergy^2 \cdot R^2$.
In fact, both factors heavily rely on the transformation rule of $\dict$, which may change substantially from application to application.
We shall return to this issue in Section~\ref{sec:applications} where we further discuss the practical relevance of our theoretical findings and investigate some desirable properties (and modifications) of the feature dictionary $\dict$.

\subsection{Contributions, Related Work, and Expected Impact}

One major goal of this work is to develop a further understanding of why even standard approaches, such as the Lasso or logistic regression, perform surprisingly well in many real-world applications.
Indeed, our main results show that successful feature selection can be already achieved with estimators which only take the raw\footnote{For simplicity, we will still speak of ``raw'' data if it has undergone some common preprocessing, like smoothing or standardization, which is not part of the actual selection procedure.} sample pairs $\{(\data_i, y_i)\}_{1 \leq i \leq m}$ as input and do not require any specific information about the data.
Regarding the illustrative setup of Theorem~\ref{thm:intro:sparseselection}, we shall also see in Section~\ref{sec:results} that the underlying model assumptions can be further generalized, including noisy non-linear outputs, strictly convex loss functions, and arbitrary (convex) signal structures.
In this sense, our framework gives fairly general solutions to the initial challenges of \ref{question:intro:theoretical} and \ref{question:intro:setup}.

As already sketched in Subsection~\ref{subsec:intro:keyidea}, our key idea is to mimic the true observation model within the data domain.
Using the novel concept of \emph{optimal dictionary representations} (cf. Definition~\ref{def:results:dictionary:optimalrepr}), it is in fact possible to obtain an approximation of $\trusig$ while the actual estimator works in the data domain.
The proofs of our results are based on the recent work of \cite{genzel2016estimation}, where the first author has provided an abstract toolbox for signal recovery in high dimensions.
To a certain extent, this paper also continues the philosophy of \cite{genzel2016estimation}, in the sense that the unknown data decomposition can be seen as an additional source of model uncertainty. 
But we will try to keep this exposition as self-contained as possible and restate the used main principles of \cite{genzel2016estimation}.

Most approaches in classical learning theory aim to model the posterior distribution $\mean[\yrnd|\datarnd]$ (conditional expectation) in a direct way.
In our context, this would basically mean that a model of the form $y_i \approx \sp{\data_i}{\tru\fv}$ is assumed to be the ground-truth output rule.
There has been a remarkable effort during the last decades to prove recovery guarantees (for $\tru\fv$) if the data variables $\data_i$ are not ``too degenerated.''
Prominent examples are the \define{restricted isometry property}, \define{restricted eigenvalue condition}, or \define{irrepresentability condition}, which require that the $\data_i$ are sufficiently well-behaved; see \cite{foucart2013cs,hastie2015sparsity,buhlmann2011statistics} for overviews.
But unfortunately, real-world data is typically highly redundant, so that these assumptions are not even met in relatively simple cases, such as mass spectra (see Figure~\ref{fig:intro:motivation:msdata}).
For that reason, various strategies have been recently suggested in the literature to deal with redundant (or almost perfectly correlated) features, for instance, \define{hierarchical clustering} (\cite{buehlmann2013clustering}) or \mbox{\define{OSCAR/OWL}} (\cite{bondell2008oscar,figueiredo2014owl}).
These methods are however mostly adapted to very specific settings and therefore do not allow for a general treatment of our problem setup.
In contrast, we shall not focus on designing sophisticated algorithms, but rather propose a novel perspective on the challenge of variable selection:
As pointed out above, one may easily circumvent the issue of redundancy by considering a hidden measurement process $y_i \approx \sp{\lat_i}{\trusig}$ and relating it to the data prior $\data_i$ in a separate step.
Such a combination of models has been much less studied theoretically, and to the best of our knowledge, this work provides the first rigorous mathematical result in this direction.
Hence, we hope that our key techniques, primarily the idea of optimal representations, could also have a certain impact on future developments in the field of feature extraction.

An alternative branch of research aims at extracting the signal variables $\lat_1, \dots, \lat_m$ directly from the raw data $\data_1, \dots, \data_m$. This could be done by applying an appropriate \define{feature map} (to the data), which is typically obtained by \define{dictionary learning} or a \define{factor analysis}.
Afterwards, the question of \ref{question:intro:theoretical}--\ref{question:intro:setup} would be essentially superfluous, since any estimator could now explicitly invoke the hidden output rule.
But on the other hand, such a strategy is only feasible as long as sufficiently many samples are available and the data are not too noisy---and these are two assumptions which are hardly satisfied in practice. For example, a \emph{peak detection} for MS~data might be extremely unstable if some of the peaks are ``buried'' in the baseline noise.
It was already succinctly emphasized by Vapnik in \cite[p.~12]{vapnik1998learning} that a direct approach should be always preferred:
\begin{highlight}
``If you possess a restricted amount of information for solving some problem, try to solve the problem directly and never solve the more general problem as an intermediate step. It is possible that the available information is sufficient for a direct solution but is insufficient for solving a more general intermediate problem.''
\end{highlight}
This fundamental principle is precisely reflected by our findings, showing that selecting features and learning a feature map can (and should) be considered as separate tasks.

The generality of our framework, in particular that $\dict$ can be arbitrary, comes with the drawback that we cannot always make a significant statement on \ref{question:intro:practical}. If no further information on the dictionary $\dict$ is given, it is virtually impossible to assess the practical relevance of an estimated feature vector $\solu\fv$.
As a consequence, when establishing a novel result for a specific application, one always needs to carefully analyze the intrinsic structure of $\dict$.
This will be illustrated for MS~data in Subsection~\ref{subsec:applications:msdata}, but we believe that our theorems may be applicable to other types of data as well, such as microarrays or hyperspectral images.
For that matter, some general rules-of-thumbs are provided in Subsection~\ref{subsec:applications:relevance}, indicating whether or not a successful feature extraction (by standard estimators) can be expected for a certain problem-of-interest.

There is also an interesting relationship of our results to \emph{compressed sensing (CS) with dictionaries} (see \cite{candes2011csdict,rauhut2008csdict} for example).
In order to apply the guarantees from \cite{genzel2016estimation}, the convex optimization program (e.g., \eqref{eq:intro:lasso}) is actually transformed into a \define{synthesis formulation} using the feature dictionary $\dict$ (see Subsection~\ref{subsec:proofs:selection}, particularly \eqref{eq:proofs:selection:estimatorequiv}).
However, it is not entirely clear if one could make use of this connection, since the CS-related theory strongly relies on an explicit knowledge of the dictionary.

\subsection{Outline and Notation}

The main focus of Section~\ref{sec:results} is on the questions of \ref{question:intro:theoretical} and \ref{question:intro:setup}. In this section,
we will develop the setup of the introduction further towards an abstract framework for feature selection, including more general data and observation models (Subsections~\ref{subsec:results:model} and \ref{subsec:results:dictionary}), strictly convex loss functions, and arbitrary convex signal structures (Subsection~\ref{subsec:results:estimatorsignal}).
The Subsections~\ref{subsec:results:guarantee} and \ref{subsec:results:refinements} then contain our main recovery guarantees (Theorem~\ref{thm:results:selection} and Theorem~\ref{thm:results:selection:refined}), while their proofs are postponed to Subsection~\ref{subsec:proofs:selection}.
The practical scope of our results is studied afterwards in Section~\ref{sec:applications}.
In this course, we will particularly investigate the benefit of standardizing the data (Subsection~\ref{subsec:applications:standardization}).
Moreover, the issue of \ref{question:intro:practical} is addressed again, first by returning to our prototype example of MS~data (Subsection~\ref{subsec:applications:msdata}), followed by a more general discussion (Subsection~\ref{subsec:applications:relevance}).
Some concluding remarks and possible extensions are finally presented in Section~\ref{sec:conclusion}.

Throughout this paper, we will recurrently make use of several (standard) notations and conventions which are summarized in the following list:
\begin{listing}
\item
	Vectors and matrices are usually written in boldface, whereas their entries are referenced by subscripts. Let $\x = (x_1, \dots, x_n) \in \R^n$ be a vector (unless stated otherwise, it is always understood to be a \emph{column} vector).
	The \define{support} of $\x$ is defined by
	\begin{equation}
		\supp(\x) := \{ 1 \leq j \leq n \suchthat x_j \neq 0 \},
	\end{equation}
	and its cardinality is $\lnorm{\x}[0] := \cardinality{\supp(\x)}$. For $1 \leq p \leq \infty$, the \define{$\l{p}$-norm} of $\x$ is given by
	\begin{equation}
	\lnorm{\x}[p] := \begin{cases} (\sum_{j = 1}^n \abs{x_j}^p)^{1/p}, & p < \infty, \\ \max_{1 \leq j \leq n} \abs{x_j}, & p = \infty.
	\end{cases}
	\end{equation}
	The associated \define{unit ball} is denoted by $\ball[p][n] := \{ \x \in \R^n \suchthat \lnorm{\x}[p] \leq 1 \}$ and the \define{(Euclidean) unit sphere} is $S^{n-1} := \{ \x \in \R^n \suchthat \lnorm{\x} = 1 \}$. The \define{operator norm} of a matrix $\vec{M} \in \R^{n'\times n}$ is defined as $\norm{\vec{M}} := \sup_{\x \in S^{n-1}} \lnorm{\vec{M}\x}$.
\item
	Let $Z$ and $Z'$ be two real-valued random variables (or random vectors). The \define{expected value} of $Z$ is denoted by $\mean[Z]$ and the \define{variance} by $\var[Z]$. Similarly, we write $\mean[Z|Z']$ for the \define{conditional expectation} of $Z$ with respect to $Z'$.
	The \define{probability} of an event $A$ is denoted by $\prob[A]$.
\item
	The letter $C$ is always reserved for a constant, and if necessary, an explicit dependence of $C$ on a certain parameter is indicated by a subscript.
	More specifically, $C$ is said to be a \define{numerical} constant if its value is independent from all involved parameters in the current setup.
	In this case, we sometimes simply write $A \lesssim B$ instead of $A \leq C \cdot B$; and if $C_1 \cdot A \leq B \leq C_2 \cdot A$ for numerical constants $C_1, C_2 > 0$, we use the abbreviation $A \asymp B$.
\item
	The phrase ``with high probability'' means that an event arises at least with a fixed (and high) \emph{probability of success}, for instance, 99\%. Alternatively, one could regard this probability as an additional parameter which would then appear as a factor somewhere in the statement. But for the sake of convenience, we will usually omit this explicit quantification.
\end{listing}

\section{General Framework for Feature Selection and Theoretical Guarantees}
\label{sec:results}

In this section, we shall further extend the exemplary setup of the introduction, ultimately leading to an abstract framework for feature selection from real-world data.
Our main results, Theorem~\ref{thm:results:selection} and Theorem~\ref{thm:results:selection:refined}, will show that rigorous guarantees even hold in this more advanced situation.
These findings particularly provide relatively general solutions to the issues of \ref{question:intro:theoretical} and \ref{question:intro:setup}.

\subsection{Model Assumptions and Problem Formulation}
\label{subsec:results:model}

For an overview of the notations introduced in this and the subsequent subsection, the reader may also consider Table~\ref{tab:results:notation} below.

Let us start by generalizing the simple binary model of \eqref{eq:intro:binarymodel}.
The following observation scheme was already considered in \cite{genzel2016estimation,plan2015lasso,plan2014highdim}:
\begin{properties}[3em]{M}
\setcounter{enumi}{0}
\item\label{assump:model:obsmodel}
	We assume that the \define{observation variables} $y_1, \dots, y_m \in \R$ obey a \define{semiparametric single-index model}
	\begin{equation}\label{eq:results:model:obsmodel}
		y_i = \fobs(\sp{\lat_i}{\trusig}), \quad i = 1, \dots, m,
	\end{equation}
	where $\trusig \in \R^p$ is the ground-truth \define{signal vector}. The \define{signal variables} $\lat_1, \dots, \lat_m \distributed \Normdistr{\vnull}{\Idm{p}}$ are independent samples of a standard Gaussian vector and $\fobs \colon \R \to \Y$ is a (possibly random) function which is independent of $\lat_i$,\footnote{The randomness of $\fobs$ is understood observation-wise, i.e., for every sample $i \in \{1, \dots, m\}$, we take an \emph{independent} sample of $\fobs$. But this explicit dependence of $\fobs$ on $i$ is omitted here.} with $\Y$ denoting a closed subset of $\R$.
\end{properties}
The range $\Y$ of $\fobs$ defines the \define{observation domain}, restricting the possible values of $y_1, \dots, y_m$. For example, we would have $\Y = \{-1,0,+1\}$ for the binary classification rule in \eqref{eq:intro:binarymodel}. The function $\fobs$ plays the role of a \define{non-linearity} in \eqref{eq:results:model:obsmodel}, modifying the linear output $\sp{\lat_i}{\trusig}$ in a certain (random) way.\footnote{One should be aware of the fact that $f$ could be also a (noisy) \emph{linear} function, even though we will continue to speak of ``non-linearities.''} Remarkably, this perturbation might be even \emph{unknown} so that, in particular, we do not need to assume any knowledge of the noise structure of the observations.

Our extension of the factor model in \eqref{eq:intro:datamodel} looks as follows:
\begin{properties}[3em]{M}
\setcounter{enumi}{1}
\item\label{assump:model:datamodel}
	The sample data are generated from a \define{linear model} of the form
	\begin{equation}\label{eq:results:model:datamodel}
		\data_i = \sum_{k = 1}^p s_{i,k} \atom_k + \sum_{l = 1}^{q} n_{i,l} \atomnoise_l \in \R^d, \quad i = 1, \dots, m,
	\end{equation}
	where the signal variables $\lat_i = (s_{i,1}, \dots, s_{i,p})$ are already determined by \ref{assump:model:obsmodel}. The \define{noise variables} $\latnoise_i = (n_{i,1}, \dots, n_{i,q}) \distributed \Normdistr{\vnull}{\Idm{q}}$, $i = 1, \dots, m$, are independently drawn from a standard Gaussian vector, which is also independent from the $\lat_1, \dots, \lat_m$. As before, $\atom_1, \dots, \atom_p \in \R^d$ are fixed (deterministic) \define{feature atoms} whereas the fixed vectors $\atomnoise_1, \dots, \atomnoise_{q} \in \R^d$ are called \define{noise atoms}.
\end{properties}
Note that the assumptions of $\lat_i$ and $\latnoise_i$ having mean zero is no severe restriction, since the input data is typically centered in advance (see Remark~\ref{rmk:applications:standardization}\ref{rmk:applications:standardization:empirical}).
Similarly, unit variances can be also taken for granted as the actual energy of the signals is determined by their feature and noise atoms, respectively.
The main novelty of \eqref{eq:results:model:datamodel}, compared to \eqref{eq:intro:datamodel}, is obviously the generalized noise term. Indeed, if $q = d$ and $\atomnoise_1, \dots, \atomnoise_d$ are (scalar multiples of) the Euclidean unit vectors of $\R^d$, we would precisely end up with the entry-wise noise structure of \eqref{eq:intro:datamodel}.
Put simply, the data $\data_1, \dots, \data_m$ are built up of a linear combination of atoms where their respective contributions (scalar factors) are random.
But once the observation process \ref{assump:model:obsmodel} is taken into account, a distinction between \emph{feature} and \emph{noise} atoms becomes meaningful:
The signal factors $s_{i,1}, \dots, s_{i,p}$ are precisely those contributing to \eqref{eq:results:model:obsmodel} whereas the noise factors $n_{i,1}, \dots, n_{i,q}$ are irrelevant.
It would be therefore also quite natural to refer to $\lat_i$ and $\latnoise_i$ as \define{active} and \define{inactive} variables, respectively.
This issue will become important again in the following subsection when we introduce the so-called \define{extended signal domain}, unifying both types of signals.

Finally, let us briefly restate the major challenge of feature selection:
\highlighted{Find a robust and accurate estimator $\solu\sig$ of the signal vector $\trusig$, using only a (small) collection of sample pairs $\{(\data_i, y_i)\}_{1 \leq i \leq m} \in \R^d \times \Y$.}
At this point, we would like to call attention to the nomenclature used for our framework: All phrases containing the word ``feature'' are associated with the \define{data domain} $\R^d$; in particular, each single component of $\data_i = (x_{i,1}, \dots, x_{i,d})$ is called a \define{feature variable}.
In contrast, the word ``signal'' is naturally related to the \define{signal domain} $\R^p$.
Hence, it is a bit inconsistent to speak of ``feature selection'' in this context because our task is rather to select signal variables.
But using a phrase like ``signal selection'' would be quite uncommon, and moreover, the output of the estimator (e.g., the Lasso) is actually a \emph{feature vector}.

\begin{remark}
In the language of statistical learning, \eqref{eq:results:model:obsmodel} is a typical example of a \define{discriminative model} that describes the \define{posterior} $\mean[y_i|\lat_i]$, 
and the above goal would translate into learning the \define{model parameters} $\trusig$.
Together with the \define{prior distribution} of the data $\data_i$ defined in \eqref{eq:results:model:datamodel}, the joint probability distribution of a sample pair $(\data_i, y_i)$ is completely determined.
In fact, any Gaussian vector can be written in terms of \eqref{eq:results:model:datamodel} with appropriately chosen feature atoms.
The actual limitation of our model are the assumptions in \ref{assump:model:obsmodel}; it is not always true that an observation $y_i$ only depends on a \emph{linear} projection of $\lat_i$. For example, one could think of an \define{additive model} where the entries of $\lat_i$ are modified in a non-linear fashion before projecting.
Fortunately, we will see in Subsection~\ref{subsec:results:refinements} that \ref{assump:model:obsmodel} can be further relaxed. 
The true observations might be even affected by some arbitrary, possibly deterministic noise. 
And furthermore, it will turn out that low correlations between the signal variables do not cause any difficulties, i.e., we can allow for $\lat_i \distributed \Normdistr{\vnull}{\Covmatr}$ with a positive definite covariance matrix $\Covmatr \in \R^{p\times p}$.
\qedrmkhere
\end{remark}

\subsection{Optimal Dictionary Representations and Extended Signal Domain}
\label{subsec:results:dictionary}

Now, we would like to adapt the key concept of \emph{optimal representations} that was already outlined in Subsection~\ref{subsec:intro:keyidea}.
In order to carry out a precise analysis, one clearly needs to incorporate the influence of the noise variables, which was essentially disregarded in \eqref{eq:intro:trudatamodel}.
As foreshadowed in the previous subsection, it is quite natural here to view the noise term as an additional (physical) source of signals.
Mathematically, this corresponds to artificially extending the signal domain by all noise variables so that we actually consider joint vectors $(\lat_i, \latnoise_i) \in \R^{p+q}$.
Thus, collecting the feature and noise atoms as matrices $\atoms := \matr{\atom_1 \mid \dots \mid \atom_p} \in \R^{d\times p}$ and $\atomsnoise := \matr{\atomnoise_1 \mid \dots \mid \atomnoise_{q}} \in \R^{d \times q}$, we obtain a simple factorization of our data model \eqref{eq:results:model:datamodel}:
\begin{equation}\label{eq:results:dictionary:factorization}
	\data_i = \sum_{k = 1}^p s_{i,k} \atom_k + \sum_{l = 1}^{q} n_{i,l} \atomnoise_l = \atoms \lat_i + \atomsnoise \latnoise_i = \matr{\atoms \mid \atomsnoise}  \matr{\lat_i \\ \latnoise_i}.
\end{equation}

A major difficulty is that this factorization is usually \emph{unknown}, forcing us to work in the data domain.
For this reason, we shall consider a \define{representing feature vector} $\tru\fv \in \R^d$ that tries to ``mimic the properties`` of the ground-truth signal vector $\trusig$.
As a first step, let us use \eqref{eq:results:dictionary:factorization} to redo the computation of \eqref{eq:intro:fromddtosd} for our generalized models \ref{assump:model:obsmodel} and \ref{assump:model:datamodel}:
\begin{align}
	y_i' :={} & \fobs(\sp{\data_i}{\tru\fv}) = \fobs\Big(\sp[\Big]{\matr{\atoms \mid \atomsnoise}\matr{\lat_i \\ \latnoise_i}}{\tru\fv}\Big) \\
	={} & \fobs(\underbrace{\sp{\lat_i}{\atoms^\T \tru\fv}}_{=: \tru{s}} + \underbrace{\sp{\latnoise_i}{\atomsnoise^\T \tru\fv}}_{=: \tru{n}}) = \fobs(\tru{s} + \tru{n}), \quad i = 1, \dots, m. \label{eq:results:dictionary:ddrepresentation}
\end{align}
Note that the dependence of the \define{signal term} $\tru{s}$ and the \define{noise term} $\tru{n}$ on the index $i$ was omitted for the sake of convenience, since all samples are identically and independently distributed.

The additional summand $\tru{n}$ in \eqref{eq:results:dictionary:ddrepresentation} might generate a certain \emph{model mismatch} with respect to \eqref{eq:results:model:obsmodel}, which means that we have $y_i \neq y_i'$ for (some of) the samples $i = 1, \dots, m$.
Our primal goal is therefore to ensure that the approximate observation rule of \eqref{eq:results:dictionary:ddrepresentation} matches as closely as possible with the true model of \eqref{eq:results:model:obsmodel}. In other words, we wish to choose $\tru\fv$ in such a way that the deviation between $y_i$ and $y_i'$ becomes very small.
Following the argumentation of Subsection~\ref{subsec:intro:keyidea}, it is promising to pick some $\tru\fv$ with $\trusig = \atoms^\T \tru\fv$, which implies that $\sp{\lat_i}{\trusig} = \sp{\lat_i}{\atoms^\T \tru\fv} = \tru{s}$.
The matrix $\dict = \atoms^\T \in \R^{p\times d}$ should be again regarded as a \define{feature dictionary}, containing the fixed \emph{intrinsic} parameters of our data model.
To impose some additional structure on $\trusig$, we may assume that the representing feature vector just belongs to a known convex \define{coefficient set} (or \define{feature set}) $\sset \subset \R^d$.
More precisely, we choose
\begin{equation}\label{eq:results:dictionary:fvstructured}
	\tru\fv \in \sset_{\trusig} := \{ \fv \in \R^d \suchthat \scalfac\fv \in \sset \text{ and } \trusig = \dict \fv \} \subset \R^d,
\end{equation}
where $\scalfac$ is constant scaling factor that depends on $\fobs$ and is specified later on in \eqref{eq:results:estimatorsignal:modelparam:scalfac}. This particularly means that $\scalfac\trusig \in \dict \sset$.
The purpose of the assumption \eqref{eq:results:dictionary:fvstructured} is twofold: On the one hand, one would like to reduce the \emph{complexity} of a potential estimator in order to avoid \emph{overfitting}, which is especially important when $m \ll d$. Our main result in Subsection~\ref{subsec:results:guarantee} will show that such a restriction of the solution space is immediately related to the number of required samples.
On the other hand, we may intend to incorporate some heuristic constraints by an appropriate choice of $\sset$.
For example, this could be sparsity, which we have already applied in Theorem~\ref{thm:intro:sparseselection} (with $\sset = R \ball[1][d]$).

Taking \eqref{eq:results:dictionary:fvstructured} into account, we have $y_i = \fobs(\sp{\lat_i}{\trusig}) = \fobs(\tru{s})$ and $y_i' = \fobs(\tru{s} + \tru{n})$.
Our goal is now to match $y_i'$ with $y_i$ by decreasing the impact of the additive noise term $\tru{n} = \sp{\latnoise_i}{\atomsnoise^\T \tru\fv}$.
Since $\latnoise_i \distributed \Normdistr{\vnull}{\Idm{q}}$, we observe that $\tru{n} \distributed \Normdistr{0}{\tru\stddev^2}$ with $\tru\stddev := \lnorm{\atomsnoise^\T \tru\fv}$.
Hence, it is quite natural to consider precisely those $\tru\fv \in \sset_{\trusig}$ that minimize the variance of $\tru{n}$: 
\begin{definition}\label{def:results:dictionary:optimalrepr}
A feature vector $\tru\fv \in \R^d$ satisfying
\begin{equation}\label{eq:results:dictionary:optimalrepr}
	\tru\fv = \argmin_{\fv \in \sset_{\trusig}} \lnorm{\atomsnoise^\T \fv}^2
\end{equation}
is called an \define{optimal representation of $\trusig$ by the feature dictionary $\dict$}.\footnote{One does not have to be concerned about the (non-)uniqueness of the minimizer in \eqref{eq:results:dictionary:optimalrepr}, since all results of this work hold for \emph{any} optimal representation.}
The associated \define{noise variance} is denoted by $\tru\stddev^2 = \lnorm{\atomsnoise^\T \tru\fv}^2$.
\end{definition}

We would like to emphasize that our concept of optimal representations is also underpinned by the fact that it provides the best possible \define{signal-to-noise ratio} (\define{SNR}) for the corrupted observation model \eqref{eq:results:dictionary:ddrepresentation}. Indeed, the SNR of \eqref{eq:results:dictionary:ddrepresentation} can be defined as
\begin{equation}\label{eq:results:dictionary:snr}
	\SNR := \frac{\var[\tru{s}]}{\var[\tru{n}]} = \frac{\var[\sp{\lat_i}{\atoms^\T \tru\fv}]}{\var[\sp{\latnoise_i}{\atomsnoise^\T \tru\fv}]} = \frac{\lnorm{\atoms^\T \tru\fv}^2}{\lnorm{\atomsnoise^\T \tru\fv}^2} =  \frac{\lnorm{\trusig}^2}{\tru\stddev^2},
\end{equation}
where the signal vector $\trusig$ is always assumed to be fixed. Therefore, maximizing the SNR is actually equivalent to minimizing $\tru\stddev^2$.

\begin{remark}\label{rmk:results:dictionary}
\begin{rmklist}
\item
	The previous paragraph has shown that our choice of $\tru\fv$ is optimal with respect to the SNR.
	But it is not entirely clear whether such a notion always leads to the ``best possible'' representation.
	The above argumentation is somewhat heuristic in the sense that we did not specify how the distance between the outputs $y_i$ and $y_i'$ is measured.
	This will be made precise in the course of our main results, namely in \eqref{eq:results:guarantee:noiseparam}.
	A more general approach towards optimal representations is briefly discussed in the concluding part of Section~\ref{sec:conclusion}.
\item
	The noise term $\tru{n}$ in \eqref{eq:results:dictionary:ddrepresentation} generates a certain perturbation of the true observation model from \ref{assump:model:obsmodel}.
	From the perspective of signal recovery, this means that we aim to reconstruct a signal $\trusig$ from samples $y_1, \dots, y_m$ while the available measurement process is inaccurate.
	Indeed, we are only allowed to invoke $y_i' = \fobs(\sp{\data_i}{\tru\fv}) = \fobs(\sp{\lat_i}{\trusig} + \tru{n})$, but not $y_i = \fobs(\sp{\lat_i}{\trusig})$.
	Consequently, one can view the noise term $\tru{n}$ as another source of uncertainty that affects the underlying measurement rule---and by our specific choice of $\tru\fv$ in Definition~\ref{def:results:dictionary:optimalrepr}, we try to minimize its impact.
\item\label{rmk:results:dictionary:extddict}
	When stating our main recovery result in Subsection~\ref{subsec:results:guarantee}, it will be very helpful to be aware of the following alternative strategy, which however leads exactly to Definition~\ref{def:results:dictionary:optimalrepr}.
	For this, let us take the perspective that was already suggested in the previous subsection:
	Instead of regarding the noise factors $\latnoise_i =(n_{i,1}, \dots, n_{i,q})$ as separate variables, one can rather interpret them as additional signal variables that remain \emph{inactive} in the observation process \eqref{eq:results:model:obsmodel}.
	This motivates us to consider the so-called \define{extended signal domain} $\R^{p+q}$ which serves as the ambient space of the joint factors $\extd{\lat}_i := (\lat_i, \latnoise_i) \distributed \Normdistr{\vnull}{\Idm{p+q}}$.\footnote{Hereafter, we agree with the convention that objects which are associated with the extended signal space are usually endowed with a tilde.}
	
	In order to adapt our initial problem formulation to this setup, we first extend the signal vector $\trusig$ in a trivial way by $\tru{\extd\sig} := (\trusig, \vnull) \in \R^{p+q}$, leading to an ``extended'' observation scheme
	\begin{equation}\label{eq:results:dictionary:obsmodelextd}
		y_i = \fobs(\sp{\lat_i}{\trusig}) = \fobs\Big(\sp[\Big]{\matr{\lat_i \\ \latnoise_i}}{\matr{\trusig \\ \vnull}}\Big) = \fobs(\sp{\extd{\lat}_i}{\tru{\extd\sig}}), \quad i = 1, \dots, m.
	\end{equation}
	Similarly, let us introduce the \define{extended dictionary}
	\begin{equation}
		\extd\dict := \matr{\dict \\ \dictnoise} = \matr{\atoms \mid \atomsnoise}^\T = \matr{\atom_1 \mid \dots \mid \atom_p \mid \atomnoise_1 \mid \dots \mid \atomnoise_q}^\T \in \R^{(p+q)\times d},
	\end{equation}
	where $\dictnoise := \atomsnoise^\T$ is referred to as the \define{noise dictionary}.

	Since $\data_i = \extd\dict^\T \extd{\lat}_i$, the computation of \eqref{eq:results:dictionary:ddrepresentation} can be rewritten as follows:
	\begin{equation} \label{eq:results:dictionary:eddrepresentation}
		y_i' = \fobs(\sp{\data_i}{\tru\fv}) = \fobs(\sp{\extd\dict^\T \extd{\lat}_i}{\tru\fv}) = \fobs(\sp{\extd{\lat}_i}{\extd\dict\tru\fv}), \quad i = 1, \dots, m.
	\end{equation}
	Compared to \eqref{eq:results:dictionary:ddrepresentation}, we are not concerned with an additive noise term here anymore, but this comes along with the drawback that there might not exist an \emph{exact} representation of the extended signal vector $\tru{\extd\sig}$ by $\extd\dict$.
	Indeed, if $p + q > d$, the matrix $\extd\dict$ is not surjective, so that it could be impossible to find $\tru\fv \in \R^d$ with $\tru{\extd\sig} = \extd\dict \tru\fv$.
	Restricting again to $\tru\fv \in \sset_{\trusig}$, we can therefore immediately conclude that the vector
	\begin{equation}
		\tru{\extd\sig}' := \extd\dict \tru\fv = \matr{\dict \tru\fv \\ \dictnoise \tru\fv} = \matr{\trusig \\ \dictnoise \tru\fv}
	\end{equation}
	has a non-vanishing noise component, i.e., $\dictnoise \tru\fv = \atomsnoise^\T \tru\fv \neq \vnull$.
	In order to match \eqref{eq:results:dictionary:eddrepresentation} with \eqref{eq:results:dictionary:obsmodelextd} closely, we can try to minimize the distance between $\tru{\extd\sig}$ and $\tru{\extd\sig}'$:
	\begin{equation}
		\tru\fv = \argmin_{\fv \in \sset_{\trusig}} \lnorm{\tru{\extd\sig} - \extd\dict\fv} = \argmin_{\fv \in \sset_{\trusig}} \lnorm[auto]{\smallmatr{\trusig \\ \vnull} - \smallmatr{\trusig \\ \atomsnoise^\T \fv}} = \argmin_{\fv \in \sset_{\trusig}} \lnorm{\atomsnoise^\T \fv}.
	\end{equation}
	Not very surprisingly, this choice precisely coincides with the finding of \eqref{eq:results:dictionary:optimalrepr}.
	
	It will turn out that the proofs of our main results in Subsection~\ref{subsec:proofs:selection} do strongly rely on the idea of working in the extended signal domain.
	This particularly explains why the extended dictionary $\extd\dict$ explicitly appears in the statements of Theorem~\ref{thm:results:selection} and Theorem~\ref{thm:results:selection:refined}.
	\qedrmkhere
\end{rmklist}
\end{remark}

\begin{table}
\centering
\tabulinesep=.5mm
\begin{tabu}{|X[l,m]|X[l,m]|}
\hline\multicolumn{2}{|c|}{\textbf{Signal domain (SD)} $\R^p$} \\ \hline
Signal factors/variables, \newline latent factors/variables (random) & $\lat_i = (s_{i,1}, \dots, s_{i,p}) \distributed \Normdistr{\vnull}{\Idm{p}}$, $i = 1, \dots, m$ \\ \hline
Signal vector, parameter vector, classifier & $\trusig, \solu\sig \in \R^p$ \\ \hline
\hline\multicolumn{2}{|c|}{\textbf{Data domain (DD)} $\R^d$} \\ \hline
Sample data, input data, feature variables (random) & $\data_1, \dots, \data_m \in \R^d$ \\ \hline
Noise factors/variables (random) & $\latnoise_i = (n_{i,1}, \dots, n_{i,q}) \distributed \Normdistr{\vnull}{\Idm{q}}$, $i = 1, \dots, m$ \\ \hline
Feature atoms & $\atom_1, \dots, \atom_p \in \R^d$, $\atoms = \smallmatr{\atom_1 \mid \dots \mid \atom_p} \in \R^{d\times p}$ \\ \hline
Noise atoms & $\atomnoise_1, \dots, \atomnoise_q \in \R^d$, $\atomsnoise = \smallmatr{\atomnoise_1 \mid \dots \mid \atomnoise_{q}} \in \R^{d \times q}$ \\ \hline
\hline\multicolumn{2}{|c|}{\textbf{Optimal representations and extended signal domain (ESD)} $\R^{p+q}$} \\ \hline
Feature dictionary & $\dict = \atoms^\T \in \R^{p\times d}$ \\ \hline
Noise dictionary & $\dictnoise = \atomsnoise^\T \in \R^{q\times d}$ \\ \hline
Feature vector, coefficient vector & $\tru\fv, \solu\fv \in \R^d$ \\ \hline
Coefficient set, feature set & $\sset \subset \R^d$ convex and bounded \\ \hline
Signal term (random) & $\tru{s} = \sp{\lat_i}{\atoms^\T \tru\fv} \distributed \Normdistr{0}{\lnorm{\trusig}^2}$ \\ \hline
Noise term (random) & $\tru{n} = \sp{\latnoise_i}{\atomsnoise^\T \tru\fv} \distributed \Normdistr{0}{\tru\stddev^2}$ \\ \hline
Noise variance & $\tru\stddev^2 = \lnorm{\atomsnoise^\T \tru\fv}^2$ \\ \hline
Extended signal factors/variables (random) & $\extd{\lat}_i = (\lat_i, \latnoise_i) \distributed \Normdistr{\vnull}{\Idm{p+q}}$, $i = 1, \dots, m$ \\ \hline
Extended signal vector & $\tru{\extd\sig} = (\trusig, \vnull) \in \R^{p+q}$  \\ \hline
Extended dictionary & $\extd\dict = \smallmatr{\dict \\ \dictnoise} \in \R^{(p+q)\times d}$ \\ \hline
\hline\multicolumn{2}{|c|}{\textbf{Observation domain (OD)} $\Y \subset \R$} \\ \hline
Observation variables, observations, measurements, outputs (random) & $y_1, \dots, y_m \in \Y$ \\ \hline
Non-linearity (i.i.d. random function) & $\fobs\colon \R \to \Y$ \\ \hline
\end{tabu}
\caption{A summary of the notions introduced in Subsections~\protect\ref{subsec:results:model} and \protect\ref{subsec:results:dictionary}.}
\label{tab:results:notation}
\end{table}

\subsection{Generalized Estimators, Model Parameters, and Effective Dimension}
\label{subsec:results:estimatorsignal}

Now, we aim to generalize the Lasso-estimator \eqref{eq:intro:lasso} which was considered in the introduction.
The first question one may ask is whether there is a compelling reason why the \define{square loss}
\begin{equation}
\losssq(\sp{\data_i}{\fv}, y_i) := \tfrac{1}{2}(\sp{\data_i}{\fv} - y_i)^2
\end{equation}
is applied to fit a (non-)linear observation model, such as \ref{assump:model:obsmodel}.
At least when some properties of the model are known, for example that $\fobs \colon \R \to \Y$ produces \emph{binary} outputs, it might be more beneficial to replace $\losssq$ by a specifically adapted loss function. There would be (empirical) evidence in the binary case that a least-square fit is outperformed by \define{logistic regression} using
\begin{equation}
	\loss^{\text{log}} (\sp{\data_i}{\fv}, y_i) :=  - y_i \cdot \sp{\data_i}{\fv} + \log(1+\exp(- y_i\cdot \sp{\data_i}{\fv})).
\end{equation}
Hence, we shall allow for a \define{general loss function} from now on,
\begin{equation}
\loss \colon \R \times \Y \to \R, \ (v,y) \mapsto \loss(v, y),
\end{equation}
which measures the \define{residual} between $v = \sp{\data_i}{\fv}$ and $y = y_i$ in a very specific way. A second extension of \eqref{eq:intro:lasso} concerns the sparsity constraint $\lnorm{\fv}[1] \leq R$. Here, we simply follow our general structural assumption of \eqref{eq:results:dictionary:fvstructured} and ask for $\fv \in \sset$.
This leads us to the following \define{generalized estimator}:
\begin{equation}
	\min_{\fv \in \R^d} \tfrac{1}{m} \sum_{i = 1}^m \loss(\sp{\data_i}{\fv}, y_i) \quad \text{subject to $\fv \in \sset$.} \label{eq:results:estimatorsignal:estimator}\tag{$P_{\loss, \sset}$}
\end{equation}
The objective functional $\lossemp[\y](\fv) := \tfrac{1}{m} \sum_{i = 1}^m \loss(\sp{\data_i}{\fv}, y_i)$ is often called the \define{empirical loss function} because it actually tries to approximate the expected loss $\mean[\loss(\sp{\data_i}{\fv}, y_i)]$.
For this reason, \eqref{eq:results:estimatorsignal:estimator} is typically referred to as an \emph{empirical structured loss minimization} in the literature.

Next, let us specify some general properties of $\loss$ that make the optimization program \eqref{eq:results:estimatorsignal:estimator} capable of signal estimation.
The following conditions, originating from \cite{genzel2016estimation}, are relatively mild and therefore permit a fairly large class of loss functions:
\begin{properties}[3em]{L}
\item\label{assump:estimator:regularity}
	\define{Regularity:} Let $\loss$ be twice continuously differentiable in the first variable. The first and second partial derivatives are then denoted by $\loss'(v,y) := \partder{\loss}{v}(v,y)$ and $\loss''(v,y) := \partder{\loss}{v}[2](v,y)$.
	Furthermore, assume that $\loss'$ is Lipschitz continuous in the second variable, i.e., there exists a constant $C_{\loss'} > 0$ such that
	\begin{equation}
		\abs{\loss'(v, y) - \loss'(v, y')} \leq C_{\loss'} \abs{y - y'} \quad \text{for all $v \in \R, \ y, y' \in \Y$.}
	\end{equation}
\item\label{assump:estimator:convexity}
	\define{Strict convexity:} Let $\loss$ be \define{strictly convex} in the first variable, i.e., there exists some continuous function $\mathcal{F}\colon \R \to \intvop{0}{\infty}$ such that $\loss''(v,y) \geq \mathcal{F}(v) > 0$ for all $(v,y) \in \R \times \Y$.\footnote{The purpose of $\mathcal{F}$ is to guarantee that, once the value of $v$ is fixed, $\loss''(v,y)$ can be bounded from below independently of $y$. If $\Y$ is bounded and $\loss''$ is continuous in the second variable, then \ref{assump:estimator:convexity} is already satisfied if $\loss''(v,y) > 0$ for all $(v,y) \in \R \times \Y$.}
\end{properties}

\begin{remark}\label{rmk:results:estimatorsignal:lossfct}
\begin{rmklist}
\item
	By condition \ref{assump:estimator:convexity}, the estimator \eqref{eq:results:estimatorsignal:estimator} becomes a convex program and efficient solvers are often available in practice.
	However, we shall not discuss computational issues and the uniqueness of solutions here.
	Fortunately, our results hold for \emph{any} minimizer $\solu\fv$ of \eqref{eq:results:estimatorsignal:estimator}, even though this might lead to different outcomes $\solu\sig = \dict\solu\fv$.
\item\label{rmk:results:estimatorsignal:lossfct:rsc}
	For the sake of simplicity, this work restricts to strictly convex loss functions, but we would like to emphasize that the framework of \cite{genzel2016estimation} contains an even more general condition on $\loss$, based on the concept of \define{restricted strong convexity} (\define{RSC}). Instead of requiring a strictly positive second derivative on the entire domain, one can show that it is actually enough to have strong convexity locally around the origin:
	\begin{equation}\label{eq:results:estimatorsignal:rsc}
		\loss''(v,y) \geq C_M \quad \text{for all $(v,y) \in \intvcl{-M}{M} \times \Y$,}
	\end{equation}
	where $C_M > 0$ is a constant and $M > 0$ is sufficiently large.
	The main results of \cite[Thm.~2.3 and Thm.~2.5]{genzel2016estimation} and their proofs indicate that such an assumption in fact seems to be a key property towards signal recovery and variable selection.
	In particular, all our theoretical guarantees do immediately generalize to loss functions that satisfy \eqref{eq:results:estimatorsignal:rsc}.
	\qedrmkhere
\end{rmklist}
\end{remark}

Probably the most remarkable feature of the generalized estimator \eqref{eq:results:estimatorsignal:estimator} is that it only takes the data pairs $\{(\data_i, y_i)\}_{1 \leq i \in m}$ as input and does neither require any explicit knowledge of \ref{assump:model:obsmodel} nor of \ref{assump:model:datamodel}.
As already highlighted in the introduction, this might become crucial in practical applications where the exact model parameters are hardly known.
But at some point of course, one needs to pay a certain price for using non-adaptive estimators. Indeed, our lack of information will impose a \emph{rescaling} of the ground-truth signal $\trusig$---we already saw this problem in \eqref{eq:intro:sparseselection:bound} and \eqref{eq:results:dictionary:fvstructured}---and the constants of the error bounds are affected as well.

In a first step, let us investigate how to quantify these model uncertainties.
The key idea is to regard the non-linearity $\fobs$ as a specific type of \emph{noise} that makes the observations deviate from a noiseless linear measurement process.
Following \cite{genzel2016estimation}, this can be captured by three \define{model parameters} $\scalfac$, $\modeldev$, and $\modeldeveta$, where $\scalfac$ is given by the solution of \eqref{eq:results:estimatorsignal:modelparam:scalfac}:\footnote{We would like to point out that this approach was originally suggested by Plan, Vershynin, and Yudovina in \cite{plan2014highdim,plan2015lasso}, where the authors introduce the same parameters for the special case of the square loss $\losssq$.}\footnote{Note that the parameter $\modeldev$ was denoted by $\sigma$ in \cite{genzel2016estimation}, but this notation could be easily mixed up with the noise variance $\tru\stddev^2$ defined previously.}
\noeqref{eq:results:estimatorsignal:modelparam:scalfac} \noeqref{eq:results:estimatorsignal:modelparam:modeldev} \noeqref{eq:results:estimatorsignal:modelparam:modeldeveta}%
\begin{subequations} \label{eq:results:estimatorsignal:modelparam}
\begin{align}
	0 = {} & \mean[\loss'(\scalfac \gaussianuniv, \fobs(\gaussianuniv)) \cdot \gaussianuniv], \label{eq:results:estimatorsignal:modelparam:scalfac} \\
	\modeldev^2 := {} & \mean[\loss'(\scalfac \gaussianuniv, \fobs(\gaussianuniv))^2], \label{eq:results:estimatorsignal:modelparam:modeldev} \\
	\modeldeveta^2 := {} & \mean[\loss'(\scalfac \gaussianuniv, \fobs(\gaussianuniv))^2 \cdot \gaussianuniv^2], \label{eq:results:estimatorsignal:modelparam:modeldeveta}
\end{align}
\end{subequations}
with $\gaussianuniv \distributed \Normdistr{0}{1}$. Since $\scalfac$ is only implicitly given, its existence is not always guaranteed.
There are in fact ``incompatible'' pairs of $\loss$ and $\fobs$ for which \eqref{eq:results:estimatorsignal:modelparam:scalfac} cannot be satisfied, see \cite[Ex.~3.4]{genzel2016estimation}.
Therefore, we shall assume for the rest of this work that the loss function $\loss$ was chosen such that $\scalfac$ exists. Fortunately, such a choice is always possible; for instance, one can isolate $\scalfac$ in \eqref{eq:results:estimatorsignal:modelparam:scalfac} when using the square loss $\losssq$.
A verification of this claim is part of the following example, which illustrates the statistical meaning of the model parameters defined in
\eqref{eq:results:estimatorsignal:modelparam}. For a more extensive discussion of the interplay between different loss functions $\loss$ and non-linearities $\fobs$, the reader is referred to \cite{genzel2016estimation}.

\begin{example}\label{ex:results:estimatorsignal:lossmodel}
\begin{rmklist}
\item\label{ex:results:estimatorsignal:lossmodel:linear}
Let us again consider the standard example of the square loss
\begin{equation}
	\losssq\colon \R \times \Y = \R \times \R \to \R, \ (v,y) \mapsto \tfrac{1}{2} (v-y)^2.
\end{equation}
Then the conditions \ref{assump:estimator:regularity}, \ref{assump:estimator:convexity} are easily verified ($(\losssq)'' \equiv 1$) and the definitions of the model parameters in \eqref{eq:results:estimatorsignal:modelparam} simplify significantly:
\noeqref{eq:results:estimatorsignal:modelparamllsq:scalfac} \noeqref{eq:results:estimatorsignal:modelparamllsq:modeldev} \noeqref{eq:results:estimatorsignal:modelparamllsq:modeldeveta}%
\begin{subequations} \label{eq:intro:modelparam}
\begin{align}
\scalfac &= \mean[\fobs(\gaussianuniv) \cdot \gaussianuniv], \label{eq:results:estimatorsignal:modelparamllsq:scalfac} \\
\modeldev^2 &= \mean[(\fobs(\gaussianuniv) - \scalfac\gaussianuniv)^2], \label{eq:results:estimatorsignal:modelparamllsq:modeldev} \\
\modeldeveta^2 &= \mean[(\fobs(\gaussianuniv) - \scalfac\gaussianuniv)^2\cdot \gaussianuniv^2]. \label{eq:results:estimatorsignal:modelparamllsq:modeldeveta}
\end{align}
\end{subequations}
Supposed that $\lnorm{\trusig} = 1$, it is very helpful to regard $\gaussianuniv := \sp{\lat_i}{\trusig} \distributed \Normdistr{0}{1}$ as a noiseless linear output and $\fobs(\gaussianuniv) = \fobs(\sp{\lat_i}{\trusig})$ as a non-linear and noisy modification of it.
In this context, we may interpret $\scalfac$ as the \emph{correlation} between these two variables, whereas $\modeldev$ and $\modeldeveta$ essentially measure their \emph{deviation}.

This intuition is particularly underpinned by the standard case of \emph{noisy linear observations} $y_i = \fobs(\sp{\lat_i}{\trusig}) := \bar{\scalfac} \sp{\lat_i}{\trusig} + \xi_i$ where $\bar{\scalfac} > 0$ is fixed and $\xi_i$ is independent, mean-zero noise. We easily compute
\begin{equation}
	\scalfac = \bar{\scalfac} \quad \text{and} \quad \modeldev^2 = \modeldeveta^2 = \mean[\xi_i^2],
\end{equation}
meaning that $\scalfac$ captures the rescaling caused by $\fobs$, and $\modeldev^2 = \modeldeveta^2$ equals the variance of the additive noise term.
Hence, it is reasonable to view the quotient $\scalfac^2 / \max\{\modeldev^2, \modeldeveta^2\}$ as the SNR of the observation model.
This stands in contrast to \eqref{eq:results:dictionary:snr}, where we were concerned with defining an SNR for the data model \ref{assump:model:datamodel}.
However, we will see in the discussion of Theorem~\ref{thm:results:selection} that both types of SNR affect the quality of our error estimates. 
\item\label{ex:results:estimatorsignal:lossmodel:binary}
Inspired by the prototype application of feature selection from MS data, one might wonder how the model parameters behave in the setup of binary outputs.
As a basic example, we may assume a random bit-flip model:
Let $y_i = \fobs(\sp{\lat_i}{\trusig}) := \xi_i \cdot \sign(\sp{\lat_i}{\trusig})$ where $\xi_i$ is an independent $\pm 1$-valued random variable with $\prob[\xi_i = 1] =: p \in \intvcl{0}{1}$. Using the square loss $\loss = \losssq$ again, one has
\begin{align}
	\scalfac = \mean[\xi_i \cdot \sign(\gaussianuniv) \cdot \gaussianuniv] = \mean[\xi_i] \mean[\abs{\gaussianuniv}] =  (2p - 1) \cdot \sqrt{\tfrac{2}{\pi}} \quad \text{and} \quad \modeldev^2 = \modeldeveta^2 =  1 - \tfrac{2}{\pi} (1-2p)^2.
\end{align}
An interesting special case is $p = \frac{1}{2}$, implying that $\scalfac = 0$. Then, $\fobs(\gaussianuniv) = \xi_i \cdot \sign(\sp{\lat_i}{\trusig})$ and $\gaussianuniv = \sp{\lat_i}{\trusig}$ are perfectly uncorrelated and there is clearly no hope for a recovery of $\trusig$ (cf. \cite[Sec.~III.A]{plan2013robust}).
Interestingly, it will turn out that accurate estimates are still possible when the chance of a bit-flip is very close to $\frac{1}{2}$.
\qedrmkhere
\end{rmklist}
\end{example}

We have already pointed out that the major objective of the structural constraint $\fv \in \sset$ in \eqref{eq:results:estimatorsignal:estimator} is to incorporate some prior knowledge, so that the size of the solution space is eventually reduced. This is highly relevant in situations with small sample-counts, where overfitting is often a serious issue.
In order to establish a measure for the complexity of signal classes, let us introduce the powerful concept of \define{(Gaussian) mean width}. This is also well-known as \define{Gaussian complexity} in statistical learning theory (cf. \cite{bartlett2003complexity}).
\begin{definition}\label{def:results:estimatorsignal:meanwidth}
	The \define{(global Gaussian) mean width} of a bounded subset $L \subset \R^n$ is given by
	\begin{equation}\label{eq:results:estimatorsignal:meanwidth}
		\meanwidth{L} := \mean[\sup_{\x \in L} \sp{\gaussian}{\x}],
	\end{equation}
	where $\gaussian \distributed \Normdistr{\vnull}{\Idm{n}}$ is a standard Gaussian random vector.
	Moreover, we call the square of the mean width $\effdim{L} := \meanwidth{L}^2$ the \define{effective dimension} of $L$.
\end{definition}
Fixing a random vector $\gaussian$, the supremum $\sup_{\x \in L} \sp{\gaussian}{\x}$ essentially measures the spatial extent (width) of $L$ in the direction of $\gaussian$, and by taking the expected value, we obtain an average measure of size (cf. \cite[Fig.~1]{genzel2016estimation}).
In general, the mean width enjoys a certain robustness against small perturbations, i.e., slightly increasing $L$ will only slightly change $\meanwidth{L}$.
A further discussion of (geometric) properties of the mean width can be found in \cite[Sec.~II]{plan2013robust}.
The following examples indicate that it is more convenient to consider the effective dimension when analyzing the complexity of signal sets.

\begin{example}\label{ex:results:estimatorsignal:effdim}
\begin{rmklist}
\item
	\emph{Linear subspaces.} Let $L \subset \R^n$ be a linear subspace of dimension $n'$. Then we have (cf. \cite{plan2015lasso})
	\begin{equation}\label{eq:results:estimatorsignal:effdim:linsubspace}
		\effdim{L \intersec \ball[2][n]} = \meanwidth{L \intersec \ball[2][n]}^2 \asymp n'.
	\end{equation}
	Note that $L$ is restricted to the Euclidean unit ball in order to obtain a bounded set.
	In fact, $\effdim{\cdot}$ measures the algebraic dimension in this case, which particularly justifies why we speak of the \emph{effective dimension} of a set.
\item\label{ex:results:estimatorsignal:effdim:polytope}
	\emph{Polytopes and finite sets.} Let $L' = \{\sig_1,\dots,\sig_k\} \subset \R^n$ be a finite collection of points. Thus, $L := \convhull(L')$ is a polytope and we have (cf. \cite[Ex.~1.3.8]{vershynin2014estimation})
	\begin{equation}\label{eq:results.estimatorsignal:effdim:polytope:bound}
		\effdim{L} = \effdim{L'} \lesssim (\max_{1 \leq j \leq k}\lnorm{\sig_j}^2) \cdot \log(k),
	\end{equation}
	where we have used that the mean width is invariant under taking the convex hull of a set (\cite[Prop.~2.1]{plan2013robust}).
	It is remarkable that the bound of \eqref{eq:results.estimatorsignal:effdim:polytope:bound} only logarithmically depend on the number of vertices, although $L$ might have full algebraic dimension.
	Therefore, polytopes with few vertices are of relatively low complexity, which makes them a good candidate for signal sets.
\item\label{ex:results:estimatorsignal:effdim:sparse}
	\emph{(Approximately) Sparse vectors.} Sparsity has emerged as one of the key model assumptions in many modern applications.
	In its simplest form, one may consider the set of \define{$s$-sparse} vectors
	\begin{equation}
		S := \{ \sig \in \R^n \suchthat \lnorm{\sig}[0] \leq s \}.
	\end{equation}
	The main difficulty when dealing with a sparse signal $\sig \in S$ is that its support is unknown, though small.
	In fact, $S$ is a union of $\binom{n}{s}$-many $s$-dimensional subspaces and we need to identify the one that contains $\sig$.
	Fortunately, this lack of knowledge affects the effective dimension only by a logarithmic factor (cf. \cite[Lem.~2.3]{plan2013robust}):\footnote{This asymptotic bound is well-established in the theory of compressed sensing and basically corresponds to the minimal number of measurements that is required to recover $s$-sparse vectors; see \cite{foucart2013cs} for example.}
	\begin{equation}\label{eq:results:estimatorsignal:effdim:sparse:boundl0}
		\effdim{S \intersec \ball[2][n]} \asymp s \log(\tfrac{2n}{s}),
	\end{equation}
	where we have again restricted to the unit ball.
	The set $S$ (or $S \intersec \ball[2][n]$) is usually not applied in practice because it would impose a non-convex constraint, and moreover, real-world signals are often not \emph{exactly} sparse. Instead, one rather tries to design a certain \define{convex relaxation}.
	For this purpose, let us apply the Cauchy-Schwarz inequality for some $\sig \in S \intersec \ball[2][n]$:
	\begin{equation}
		\lnorm{\sig}[1] \leq \sqrt{\lnorm{\sig}[0]} \cdot \lnorm{\sig} \leq \sqrt{s},
	\end{equation}
	that means, $S \intersec \ball[2][n] \subset \sqrt{s}\ball[1][n]$.
	Since $L := \sqrt{s}\ball[1][n]$ is a polytope with $2n$ vertices, we obtain by part \ref{ex:results:estimatorsignal:effdim:polytope} that
	\begin{equation}\label{eq:results:estimatorsignal:effdim:sparse:boundl1}
		\effdim{L} \lesssim s \log(2n),
	\end{equation}
	which is almost as good as \eqref{eq:results:estimatorsignal:effdim:sparse:boundl0}. 
	This observation verifies that a scaled $\l{1}$-ball, as used for the standard Lasso \eqref{eq:intro:lasso}, can serve as an appropriate convex and robust surrogate of sparse vectors.
\item\label{ex:results:estimatorsignal:effdim:dictionary}
	\define{Representations in a dictionary.}
	In this work, the most important classes of signals are those arising from \emph{dictionary representations}.
	Motivated by the previous subsection, let $\dict \in \R^{n \times d}$ be a \define{dictionary} and assume that the class-of-interest is given by $L = \dict \sset$ for a certain \define{coefficient set} $\sset \subset \R^d$.
	A basic application of \define{Slepian's inequality} \cite[Lem.~8.25]{foucart2013cs} provides a general, yet non-optimal, bound on the effective dimension:
	\begin{equation}\label{eq:results:estimatorsignal:effdim:dictionary:slepianbound}
		\effdim{L} = \effdim{\dict\sset} \lesssim \norm{\dict}^2 \cdot \effdim{\sset}.
	\end{equation}
	This estimate becomes particularly bad when the operator norm of $\dict$ is large, which is unfortunately the case for highly overcomplete dictionaries as we consider.
	However, for many choices of coefficient sets there are sharper bounds available. For example, if $\sset = \convhull\{ \sig_1, \dots, \sig_k \} \subset \ball[2][d]$ is a polytope, so is $L = \dict\sset = \convhull\{ \dict\sig_1, \dots, \dict\sig_k \}$. Then we obtain
	\begin{equation}\label{eq:results:estimatorsignal:effdim:dictionary:dictpolytope}
		\effdim{L} = \effdim{\dict\sset} \lesssim D_{\max}^2 \cdot \log(k)
	\end{equation}
	by part \ref{ex:results:estimatorsignal:effdim:polytope}, where $D_{\max} := \max_{1 \leq j \leq k} \lnorm{\dict \sig_j} \leq \norm{\dict}$.
	\qedrmkhere
\end{rmklist}
\end{example}

\subsection{Main Result}
\label{subsec:results:guarantee}

In what follows, let us accept the notations introduced in Subsections~\ref{subsec:results:model}--\ref{subsec:results:estimatorsignal}; in particular, recall Table~\ref{tab:results:notation} and the definitions of the model parameters $\scalfac$, $\modeldev$, $\modeldeveta$ in \eqref{eq:results:estimatorsignal:modelparam}.
Moreover, the feature vector $\tru\fv \in \R^d$ is always understood to be an optimal representation of $\trusig \in \R^p$ (according to Definition~\ref{def:results:dictionary:optimalrepr}).
Since the resulting representation in the data domain \eqref{eq:results:dictionary:ddrepresentation} is usually not exact, we capture this 
\emph{model mismatch} by means of the following \define{noise parameter}:
\begin{equation}\label{eq:results:guarantee:noiseparam}
	\advdev := \sqrt{\tfrac{1}{m} \sum_{i = 1}^m \abs{y_i - y_i''}^2} \geq 0,
\end{equation}
where
\begin{equation}\label{eq:results:guarantee:modddrepresentation}
	y_i'' := \fobs\Big( \tfrac{\sp{\data_i}{\tru\fv}}{\sqrt{\lnorm{\trusig}^2 + \lnorm{\dictnoise \tru\fv}^2}} \Big) = \fobs\Big( \tfrac{\tru{s} + \tru{n}}{\sqrt{\lnorm{\trusig}^2 + \tru\stddev^2}} \Big), \quad i = 1, \dots, m.
\end{equation}
Note that $\advdev$ might be a random variable, depending on both the non-linearity $\fobs$ as well as on the SNR between $\tru{s}$ and $\tru{n}$ (cf. \eqref{eq:results:dictionary:snr}).

\begin{remark}\label{rmk:results:guarantee:modddrepresentation}
The modified outputs $y_i''$ arise from rescaling $y_i' = \fobs(\tru{s} + \tru{n})$ by a factor of $\tru\scalsig := \tsqrt{\lnorm{\trusig}^2 + \tru\stddev^2}$, which guarantees that $\tfrac{\tru{s} + \tru{n}}{\tru\scalsig}$ has unit variance.
Unfortunately, such a technicality is  necessary to make the framework of \cite{genzel2016estimation} applicable (see also Remark~\ref{rmk:proofs:selection}\ref{rmk:proofs:selection:rescaledobs}).
Since the additional factor of $\tru\scalsig$ is nonessential for the derivation of Subsection~\ref{subsec:results:dictionary}, it was omitted for the sake of clarity.
But this could have been easily incorporated by replacing the feature vector $\tru\fv$ by a dilated version $\tru\scalsig^{-1}\tru\fv$.
\qedrmkhere
\end{remark}

We are now ready to state the main result of this work. Its proof is postponed to Subsection~\ref{subsec:proofs:selection}.
\begin{theorem}\label{thm:results:selection}
	Suppose that the input data pairs $\{(\data_i, y_i)\}_{1 \leq i \leq m}$ obey \ref{assump:model:obsmodel} and \ref{assump:model:datamodel}. Furthermore, let \ref{assump:estimator:regularity} and \ref{assump:estimator:convexity} hold true. We assume that $\lnorm{\trusig} = 1$ and $\trusigmu \in  \dict\sset$ for a bounded, convex coefficient set $\sset \subset \R^d$ containing the origin.
	Then there exists a constant of the form $\modeldevconst = C \cdot \max\{1, \modeldev, \modeldeveta\} > 0$ with $C > 0$ such that the following holds with high probability:\footnote{More precisely, the constant $C$ may depend on the ``probability of success'' as well as on the so-called RSC-constant of $\loss$ (cf. Remark~\ref{rmk:results:estimatorsignal:lossfct}\ref{rmk:results:estimatorsignal:lossfct:rsc}).}
	
	If the number of samples satisfies
	\begin{equation}\label{eq:results:selection:samplecount}
		m \geq C \cdot \effdim{\extd\dict\sset},
	\end{equation}
	then, setting $\solu\sig := \dict \solu\fv$ for any minimizer $\solu\fv$ of \eqref{eq:results:estimatorsignal:estimator}, we have
	\begin{equation}\label{eq:results:selection:bound}
		\lnorm{\solu\sig - \tru\snrscal\trusig} = \lnorm{\dict\solu\fv - \tru\snrscal \dict\tru\fv} \leq \modeldevconst \bigg( \Big(\frac{\effdim{\extd\dict\sset}}{m}\Big)^{1/4} + \advdev \bigg),
	\end{equation}
	where $\tru\snrscal := \frac{\scalfac}{\sqrt{1 + \tru\stddev^2}} = \frac{\scalfac}{\sqrt{1 + \lnorm{\dictnoise\tru\fv}^2}}$.
\end{theorem}
The statement of Theorem~\ref{thm:results:selection} gives a quite general answer to our initial challenge \ref{question:intro:theoretical}. 
Indeed, the error bound \eqref{eq:results:selection:bound} shows how one can use an estimated feature vector $\solu\fv$ to approximate the ground-truth signal $\trusig$.
A very remarkable observation is that the applied estimator \eqref{eq:results:estimatorsignal:estimator} only takes the raw data $\{(\data_i, y_i)\}_{1 \leq i \leq m}$ as input.
This indicates that---at least from a theoretical perspective---effective variable selection can be still achieved without any knowledge of the model specifications in \ref{assump:model:obsmodel} and \ref{assump:model:datamodel}.
However, this disregard is reflected by the fact that $\trusig$ needs to be appropriately rescaled (by $\tru\snrscal$) and the error (semi-)metric explicitly depends on the (unknown) feature dictionary $\dict$.
The above result does therefore not automatically resolve \ref{question:intro:practical}, unless the behavior of $\dict$ is well-understood or a good approximation of it is available.
We shall come back to these issues in Section~\ref{sec:applications}, where some applications and consequences of Theorem~\ref{thm:results:selection} are discussed.

Let us now analyze under which conditions \eqref{eq:results:selection:bound} establishes a meaningful error estimate:
\begin{listing}
\item
	\emph{Sample count.}
	If $m$ (greatly) exceeds $\effdim{\extd\dict\sset}$---which is already assumed by \eqref{eq:results:selection:samplecount}---the fraction $\effdim{\extd\dict\sset} / m$ becomes small.
	Thus, the effective dimension of $\extd\dict\sset$ can be regarded as a \emph{threshold} for the number of required observations.
	This explains why it is very important to impose a certain low-complexity constraint on $\trusig$ when only a few samples are available.
	A successful application of Theorem~\ref{thm:results:selection} therefore always asks for an appropriate upper bound on $\effdim{\extd\dict\sset}$.
	Later on, we will illustrate how this can be done in the case of sparse representations (see Subsection~\ref{subsec:applications:standardization} and proof of Theorem~\ref{thm:intro:sparseselection}).
\item
	\emph{Signal-to-noise ratio.}
	The accuracy of \eqref{eq:results:selection:bound} is highly sensitive to the choice of $\tru\fv$, since it affects both the scaling factor $\tru\snrscal$ and the noise parameter $\advdev$.
	By \eqref{eq:results:dictionary:optimalrepr}, we have defined $\tru\fv$ precisely such that the noise variance $\tru\stddev^2$ is minimized, i.e., $\abs{\tru\snrscal}$ is guaranteed to be as large as possible (among all representations).
	Consequently, an undesirable situation occurs if $\abs{\tru\snrscal}$ is still very small.
	Then, the rescaled vector $\tru\snrscal\trusig$ is close to $\vnull$ and \eqref{eq:results:selection:bound} provides a rather poor bound.\footnote{This can be seen more easily when the inequality \eqref{eq:results:selection:bound} is first divided by $\abs{\tru\snrscal}$.}
	Such an implication is actually not very surprising because if even an \emph{optimal} representation of $\trusig$ suffers from a low SNR, the underlying problem of variable selection might be intractable.
	
	The impact of $\advdev$, on the other hand, also depends on the (unknown) output function $\fobs$. Thus, it is quite difficult to make a general statement on its behavior.
	But at least for the popular cases of noisy linear and binary responses, we will provide simple high-probability bounds in Example~\ref{ex:applications:standardization:noiseparam}, which again depend on the SNR.
\item
	\emph{Model parameters.}
	We have illustrated in Example~\ref{ex:results:estimatorsignal:lossmodel}\ref{ex:results:estimatorsignal:lossmodel:linear} that $\scalfac$ should be seen as a measure for the expected rescaling (of $\sp{\lat_i}{\trusig}$) caused by $\fobs$.
	Since this ``scaling-effect'' is completely unknown to the estimator \eqref{eq:results:estimatorsignal:estimator}, we need to replace the actual signal-of-interest $\trusig$ by a dilated version $\scalfac\trusig$ in Theorem~\ref{thm:results:selection}.
	The necessity of such a technical step becomes already clear in the noiseless linear case when the outputs are multiplied by a fixed but unknown scalar; see again Example~\ref{ex:results:estimatorsignal:lossmodel}\ref{ex:results:estimatorsignal:lossmodel:linear} with $\xi_i \equiv 0$.
	The other two model parameters $\modeldev$ and $\modeldeveta$, capturing the ``variance'' of $\fobs$, are part of $\modeldevconst = C \cdot \max\{1, \modeldev, \modeldeveta\}$.
	Dividing \eqref{eq:results:selection:bound} by $\abs{\scalfac}$, we conclude that the resulting constant essentially scales like $\max\{1, \modeldev, \modeldeveta\} / \abs{\scalfac}$. Hence, the quality of the error estimate is improved if $\abs{\scalfac}$ is relatively large compared to $\modeldev$ and $\modeldeveta$.
	Interestingly, this observation is very closely related to our findings in Example~\ref{ex:results:estimatorsignal:lossmodel}\ref{ex:results:estimatorsignal:lossmodel:linear}, where we have identified the quotient $\scalfac / \max\{\modeldev, \modeldeveta\}$ as the SNR of \ref{assump:model:obsmodel}.
\end{listing}

Finally, we would like to emphasize that Theorem~\ref{thm:results:selection} also satisfies our desire for a general framework for feature selection, as formulated in \ref{question:intro:setup}.
The above setting in fact allows for a unified treatment of non-linear single-index models, arbitrary convex coefficient sets, and strictly convex loss functions.
Concerning the asymptotic behavior of \eqref{eq:results:estimatorsignal:estimator}, it has even turned out that the actual capability of approximating $\trusig$ neither depends on the (noisy) modifier $f$ nor on the choice of loss function $\loss$; see also \cite[Sec.~2]{genzel2016estimation} for a more extensive discussion.

\begin{remark}\label{rmk:results:selection}
\begin{rmklist}
\item
	A careful analysis of the proof of Theorem~\ref{thm:results:selection} shows that the optimality property of $\tru\fv$ from \eqref{eq:results:dictionary:optimalrepr} is never  used. Indeed, the statement of Theorem~\ref{thm:results:selection} actually holds for any representation of $\trusig$, i.e., $\tru\fv \in \sset_{\trusig}$.
	This is particularly relevant to applications where the feature vector $\tru\fv$ needs to be explicitly constructed. But in such a situation, the resulting error bound \eqref{eq:results:selection:bound} might be of course less significant, meaning that $\tru\snrscal$ decreases and $\advdev$ increases.
\item\label{rmk:results:selection:scaling}
	A certain normalization of the signal vector, e.g., $\lnorm{\trusig} = 1$, is inevitable in general. For example, it is obviously impossible to recover the magnitude of $\trusig$ if the outputs obey a binary rule $y_i = \sign(\sp{\lat_i}{\trusig})$.
	The same problem occurs for the noisy linear case of Example~\ref{ex:results:estimatorsignal:lossmodel}\ref{ex:results:estimatorsignal:lossmodel:linear} when there is an unknown scaling parameter $\std{\scalfac} > 0$ involved.
	

	Given a \emph{structured} signal $\scalfac \trusig \in \dict \sset$, the unit-norm assumption can be always achieved by considering $\trusig / \lnorm{\trusig}$ instead of $\trusig$.
	But this would imply that the coefficient set $\sset$ needs to be rescaled by $\lnorm{\trusig}^{-1}$ in order to satisfy the hypotheses of Theorem~\ref{thm:results:selection}.
	Since $\lnorm{\trusig}$ is usually unknown in real-world applications, it is very common to introduce an additional \define{scaling parameter} $R > 0$ and to use a scaled coefficient set $R \sset$ for Theorem~\ref{thm:results:selection}---note that this approach was already applied in Theorem~\ref{thm:intro:sparseselection}.
	If $R$ is sufficiently large, we can always ensure $\scalfac \trusig / \lnorm{\trusig} \in R \dict \sset$,\footnote{More precisely, we also need to assume that $\dict$ is surjective and $\sset$ is full-dimensional.} but at the same time, the effective dimension $\effdim{R \extd\dict\sset} = R^2 \cdot \effdim{\extd\dict\sset}$ starts to grow.
	Consequently, the significance of Theorem~\ref{thm:results:selection} is highly sensitive to the choice of $R$.
	Finding an appropriate scaling parameter is a deep issue in general, which arises (in equivalent forms) almost everywhere in structured and constrained optimization.
	Especially in the machine learning literature, numerous approaches have been developed to cope with this challenge; for an overview see \cite{hastie2009elements}.
	\qedrmkhere
\end{rmklist}
\end{remark}

\subsection{Some Refinements and Extensions}
\label{subsec:results:refinements}

Before proceeding with practical aspects of our framework in Section~\ref{sec:applications}, let us discuss some extensions of Theorem~\ref{thm:results:selection} which could be interesting for various real-world applications.
A first refinement concerns the limited scope of the output procedure in \ref{assump:model:obsmodel}.
Indeed, the true observations $y_1, \dots, y_m$ are supposed to \emph{precisely} follow a simple single-index model, thereby implying that one is always able to determine an appropriate parameter vector $\trusig$ for \eqref{eq:results:model:obsmodel}.
In practice however, it could happen that such a choice of $\trusig$ is actually impossible.
One would rather assume instead that there exists a vector $\trusig$, e.g., a \define{structured Bayes estimator}, which approximates the truth such that $y_i \approx \fobs(\sp{\lat_i}{\trusig})$ holds for (many) $i = 1, \dots, m$.
This motivates us to permit a certain degree of \emph{adversarial} model mismatch that could be even systematic and deterministic.
Moreover, we shall drop the assumption of independent signal factors $s_{i,1}, \dots, s_{i,p}$---a hypothesis that is rarely satisfied in applications---and allow for low correlations between them.
These desiderata give rise to the following extension of \ref{assump:model:obsmodel}:
\begin{enumerate}[label={\textsc{(M\arabic*')}},leftmargin=3em]
\item\label{assump:model:obsmodelrefined}
	We consider a single-index model
	\begin{equation}\label{eq:results:refinements:obsmodel}
		y_i^0 := \fobs(\sp{\lat_i}{\trusig}), \quad i = 1, \dots, m,
	\end{equation}
	where $\trusig \in \R^p$ and $\lat_1, \dots, \lat_m \distributed \Normdistr{\vnull}{\Covmatr}$ i.i.d.\ with positive definite \emph{covariance matrix} $\Covmatr \in \R^{p \times p}$. As before, $\fobs \colon \R \to \Y$ is a (possibly random) function which is independent of $\lat_i$, and $\Y \subset \R$ is closed.
	The true \define{observations} $y_1, \dots, y_m \in \Y$ are allowed to differ from \eqref{eq:results:refinements:obsmodel} and we assume that their deviation is bounded by means of
	\begin{equation}
		\sqrt{\tfrac{1}{m} \sum_{i = 1}^m \abs{y_i^0 - y_i}^2} \leq \tru\advdev
	\end{equation}
	for some \define{adversarial noise parameter} $\tru\advdev \geq 0$.
\end{enumerate}

With some slight modifications, the statement of Theorem~\ref{thm:results:selection} still holds in this more general setting.
For the following result, we use the same notation as in the previous subsection, but note that due to \eqref{eq:results:refinements:obsmodel}, the noise parameter $\advdev$ from \eqref{eq:results:guarantee:noiseparam} now takes the form
\begin{equation}
	\advdev = \sqrt{\tfrac{1}{m} \sum_{i = 1}^m \abs{y_i^0 - y_i''}^2}.
\end{equation}

\begin{theorem}\label{thm:results:selection:refined}
	Suppose that the input data pairs $\{(\data_i, y_i)\}_{1 \leq i \leq m}$ obey \ref{assump:model:obsmodelrefined} and \ref{assump:model:datamodel}. Furthermore, let \ref{assump:estimator:regularity} and \ref{assump:estimator:convexity} hold true. We assume that $\lnorm{\sqrt{\Covmatr}\trusig} = 1$ and $\trusigmu \in  \dict\sset$ for a bounded, convex coefficient set $\sset \subset \R^d$ containing the origin.\footnote{Since $\Covmatr \in \R^{p \times p}$ is positive definite, there exists a unique, positive definite matrix root $\sqrt{\Covmatr} \in \R^{p \times p}$ with $\sqrt{\Covmatr} \cdot \sqrt{\Covmatr} = \Covmatr$.}
	Then, with the same constant $\modeldevconst = C \cdot \max\{1, \modeldev, \modeldeveta\} > 0$ as in Theorem~\ref{thm:results:selection}, the following holds with high probability:
	
	Let the number of samples satisfy
	\begin{equation}\label{eq:results:selection:refined:samplecount}
		m \geq C \cdot \effdim{\extd\dict_{\Covmatr}\sset},
	\end{equation}
	where $\extd\dict_{\Covmatr} := \smallmatr{\sqrt{\Covmatr}\dict \\ \dictnoise}$. Then, setting $\solu\sig := \dict \solu\fv$ for any minimizer $\solu\fv$ of \eqref{eq:results:estimatorsignal:estimator}, we have
	\begin{equation}\label{eq:results:selection:refined:bound}
		\lnorm{\sqrt{\Covmatr}(\solu\sig - \tru\snrscal\trusig)} = \lnorm{\sqrt{\Covmatr}(\dict\solu\fv - \tru\snrscal \dict\tru\fv)} \leq \modeldevconst \bigg( \Big(\frac{\effdim{\extd\dict_{\Covmatr}\sset}}{m}\Big)^{1/4} + \advdev + \tru\advdev \bigg).
	\end{equation}
\end{theorem}
The proof is again postponed to Subsection~\ref{subsec:proofs:selection}; note that the statements of Theorem~\ref{thm:results:selection} and Theorem~\ref{thm:results:selection:refined} coincide if $\Covmatr = \Idm{p}$ and $\tru\advdev = 0$.
The impact of the adversarial noise is reflected by an additive error term $\tru\advdev$ in \eqref{eq:results:selection:refined:bound}.
If the mismatch between the proper model of \eqref{eq:results:refinements:obsmodel} and the true observations $y_1, \dots, y_m$ is not too large, one should be able to control the value of $\tru\advdev$. But similarly to handling $\advdev$, this also relies on the specific behavior of the non-linear modifier $\fobs$.
Another difference to Theorem~\ref{thm:results:selection} is the additional deformation of $\trusig$ and $\extd\dict$ by $\sqrt{\Covmatr}$, which is due to the non-trivial covariance structure of the latent factors $\lat_i$.
Here, we may run into problems if some of the signal variables are almost perfectly correlated with each other, so that $\sqrt{\Covmatr}$ is close to being singular.
In such a situation, a good bound on $\lnorm{\sqrt{\Covmatr}(\solu\sig - \tru\snrscal\trusig)}$ does not automatically imply that we also have $\solu\sig \approx \tru\snrscal\trusig$.
This observation underpins our intuition that, without prior information, (almost) collinear features are considered to be equally important, and it is therefore unclear which one should be selected.
Finally, it is again worth mentioning that Theorem~\ref{thm:results:selection:refined} continues the fundamental philosophy of this work in the sense that the applied estimator \eqref{eq:results:estimatorsignal:estimator} does not require any knowledge of $\Covmatr$ and $\tru\advdev$.

\begin{remark}
	Exploiting the abstract framework of \cite{genzel2016estimation}, the above results could be even further generalized. This particularly concerns the asymptotic error rate of $O(m^{-1/4})$ in \eqref{eq:results:selection:refined:bound}, which is known to be non-optimal for some choices of $\sset$ (see \cite[Sec.~4]{plan2014highdim}).
	The key step towards optimal error bounds is to introduce a \emph{localized} version of the mean width, capturing the geometric structure of $\dict\sset$ in a (small) neighborhood around the desired signal $\scalfac\trusig$.
	This approach could be easily incorporated in Theorem~\ref{thm:results:selection:refined} by using \cite[Thm.~2.8]{genzel2016estimation}.
	But since the focus of this work is rather on a mathematical theory of feature selection, we omit this extension here and leave the details to the interested reader.
	\qedrmkhere
\end{remark}

\section{Consequences and Applications}
\label{sec:applications}

In this part, we focus on the application of our theoretical framework to real-world problems.
Subsection~\ref{subsec:applications:standardization} illustrates that \emph{standardizing} the data can be helpful (and sometimes even necessary) to obtain significant error bounds from Theorem~\ref{thm:results:selection}.\footnote{In the following, we usually refer to Theorem~\ref{thm:results:selection} as our main result. But unless stated otherwise, similar conclusions hold for its generalization, Theorem~\ref{thm:results:selection:refined}, as well.}
In Subsection~\ref{subsec:applications:msdata}, we shall return to our prototype example of proteomics analysis and prove a rigorous guarantee for feature extraction from MS~data.
Finally, the general scope of our results is discussed in Subsection~\ref{subsec:applications:relevance}, including some rules-of-thumb when Theorem~\ref{thm:results:selection} implies a practice-oriented statement.

\subsection{Standardization of the Data and Signal-to-Noise Ratio}
\label{subsec:applications:standardization}

The discussion part of Theorem~\ref{thm:results:selection} has shown that the quality of the error estimate \eqref{eq:results:selection:bound} heavily relies on controlling the effective dimension $\effdim{\extd\dict \sset}$ as well as on the noise parameter $\advdev$.
For illustration purposes, let us focus on the important case of sparse representations here and assume that $\sset = R \ball[1][d]$ with some ``appropriately chosen'' scaling factor $R > 0$ (see also Remark~\ref{rmk:results:selection}\ref{rmk:results:selection:scaling} and Example~\ref{ex:results:estimatorsignal:effdim}\ref{ex:results:estimatorsignal:effdim:sparse}).
Adapting the notation from Subsection~\ref{subsec:intro:guarantee}, we denote the columns of the dictionaries $\dict$, $\dictnoise$, $\extd\dict$ by $\dictatom_j$, $\noiseatom_j$, $\extd\dictatom_j$, respectively, $j = 1, \dots, d$. Hence,
\begin{equation}
\extd\dict = \matr{\extd\dictatom_1 \mid \dots \mid \extd\dictatom_d} = \left[\begin{array}{c|c|c}\dictatom_1 & \dots & \dictatom_d \\ \noiseatom_1 & \dots & \noiseatom_d\end{array}\right] = \matr{\dict \\ \dictnoise}.
\end{equation}
At this point, it is particularly useful to think of $\extd\dictatom_j$ as a \emph{dictionary atom} that determines the contribution of each single feature/noise atom (i.e., the rows of $\extd\dict$) to the $j$-th feature variable of the data. 

Since $\ball[1][d] = \convhull\{\pm \vunit_1, \dots, \pm \vunit_d\}$, where $\vunit_1, \dots, \vunit_d \in \R^d$ are the canonical unit vectors of $\R^d$,
the estimate \eqref{eq:results:estimatorsignal:effdim:dictionary:dictpolytope} of Example~\ref{ex:results:estimatorsignal:effdim}\ref{ex:results:estimatorsignal:effdim:dictionary} yields
\begin{equation}\label{eq:applications:standardization:effdimbound}
	\effdim{\extd\dict\sset} = R^2 \cdot \effdim{\extd\dict \ball[1][d]} \lesssim R^2 \cdot \extd{D}_{\max}^2 \cdot \log(2d)
\end{equation}
with $\extd{D}_{\max} := \max_{1 \leq j \leq d} \lnorm{\extd\dictatom_j}$.
This implies that the asymptotic behavior of the effective dimension can be controlled by means of the maximal column norm of $\extd\dict$.
In order to establish a bound on $\extd{D}_{\max}$, let us first compute the variance of each feature variable of the data $\data_i = (x_{i,1}, \dots, x_{i,d})$.\footnote{Since the data is i.i.d., all results will be independent of the actual sample index $i = 1, \dots, m$.}
Using the factorization \eqref{eq:results:dictionary:factorization} and $\atoms = \dict^\T$, $\atomsnoise = \dictnoise^\T$, we obtain
\begin{equation}\label{eq:applications:standardization:variance}
	\var[x_{i,j}] = \var[\dictatom_j^\T \lat_i + \noiseatom_j^\T \latnoise_i] = \var[\sp{\dictatom_j}{\lat_i}] + \var[\sp{\noiseatom_j}{\latnoise_i}] = \lnorm{\dictatom_j}^2 + \lnorm{\noiseatom_j}^2 = \lnorm{\extd\dictatom_j}^2
\end{equation}
for $j = 1, \dots, d$.
Regarding \eqref{eq:applications:standardization:effdimbound}, it is therefore beneficial to \define{standardize} the data, meaning that $x_{i,j}$ gets replaced by $x_{i,j} / {\scriptstyle\sqrt{\var[x_{i,j}]}}$.\footnote{Here, we implicitly assume that $\var[x_{i,j}] \neq 0$. But this is no severe restriction, since a zero-variable would not be relevant to feature selection anyway.}
Every feature variable thereby achieves unit variance and the resulting dictionary atoms $\extd\dictatom_j / \lnorm{\extd\dictatom_j}$ are normalized.
And once again, it is crucial that such a simple step does only require the very raw data and can be realized without any knowledge of $\extd\dict$.

In the following, let us assume without loss of generality that $\var[x_{i,j}] = 1$  (the feasibility of such an assumption is discussed below in Remark~\ref{rmk:applications:standardization}\ref{rmk:applications:standardization:empirical}).
We then have $\extd{D}_{\max}^2 = \max_{1 \leq j \leq d} \lnorm{\extd\dictatom_j}^2 = \max_{1 \leq j \leq d} \var[x_{i,j}] = 1$, so that
\begin{equation}\label{eq:applications:standardization:effdimboundstd}
	\effdim{\extd\dict\sset} \lesssim R^2 \cdot \log(2d).
\end{equation}
This estimate does not depend on the underlying data model \ref{assump:model:datamodel}, except for the log-factor.
Hence, we can draw the remarkable conclusion that the estimator \eqref{eq:results:estimatorsignal:estimator} of Theorem~\ref{thm:results:selection} works like an \define{oracle} in the sense that it almost performs as good as if the signal factors $\lat_i$ were explicitly known. 
But note that the parameters $\tru\snrscal$ and $\advdev$ still may have a certain impact on the approximation quality, as we will analyze below.

\begin{remark} \label{rmk:applications:standardization}
\begin{rmklist}
\item\label{rmk:applications:standardization:anisotropic}
	The above argumentation can be easily adapted to Theorem~\ref{thm:results:selection:refined}, which allows for anisotropic signal variables $\lat_i \distributed \Normdistr{\vnull}{\Covmatr}$.
	To see this, one simply needs to repeat the computations with $\dict$ and $\extd\dict$ replaced by $\sqrt{\Covmatr} \dict$ and $\extd\dict_{\Covmatr}$, respectively.
\item\label{rmk:applications:standardization:empirical}
	In most practical applications, we do not have access to the random distribution of $x_{i,j}$. Thus, $\var[x_{i,j}]$ cannot be exactly computed and a standardization becomes infeasible.
	A typical way out is to perform an \define{empirical standardization} instead, which makes only use of the available sample data: For a fixed index $j \in \{1, \dots, d\}$, calculate the \define{empirical mean} $\std{x}_j := \tfrac{1}{m} \sum_{i = 1}^m x_{i,j}$ as well as the \define{empirical variance} $\std{\stddev}_j^2 := \tfrac{1}{m} \sum_{i = 1}^m (x_{i,j} - \std{x}_j)^2$. Then, the initial component $x_{i,j}$ is replaced by
	\begin{equation}
		\std{x}_{i,j} := \frac{x_{i,j} - \std{x}_j}{\std{\stddev}_j} \quad \text{for all $i = 1, \dots, m$ and $j = 1, \dots, d$.}
	\end{equation}
	Since the law of large numbers yields $\std{x}_j \to \mean[x_{i,j}]$ and $\std{\stddev}_j^2 \to \var[x_{i,j}]$ as $m \to \infty$, we can conclude that $\std{\stddev}_j^2 \approx \var[x_{i,j}] = \lnorm{\extd\dictatom_j}^2$. Consequently, an \define{empirical standardization} implies $\extd{D}_{\max} \approx 1$ for $m$ sufficiently large, so that the bound of \eqref{eq:applications:standardization:effdimboundstd} is at least approximately true.
\item\label{rmk:applications:standardization:snr}
	Due to the computational inaccuracies described in the previous remark, one could ask whether simply rescaling the data could be a better alternative to standardization.
	Unfortunately, such an approach can sometimes cause serious problems: Since $\lnorm{\trusig} = 1$ and $\trusig = \dict \tru\fv$, it might happen that several entries of $\tru\fv$ must be very large in magnitude (when the corresponding atom norm $\lnorm{\dictatom_j}$ is small) whereas others must be small (when $\lnorm{\dictatom_j}$ is large).
	The first case might involve an enlargement of the parameter $R > 0$ in order to guarantee $\scalfac \tru\fv \in \sset = R\ball[1][d]$, but at the same time, the estimate of \eqref{eq:applications:standardization:effdimbound} becomes worse.
	This issue could be easily fixed by rescaling the data $\data_i \mapsto \lambda \cdot \data_i$ with a sufficiently large $\lambda > 0$.
	However, the value of $\extd{D}_{\max}$ would then blow up (due to the large atom norms of the second case), which makes \eqref{eq:applications:standardization:effdimbound} meaningless again.
	
	Instead, a standardization aims to circumvent this drawback by adjusting each atom separately.
	But even then, we could run into difficulties as the SNR of the individual feature variables is too low. Indeed, if
	\begin{equation}\label{eq:applications:standardization:badsnr}
		\tfrac{\lnorm{\dictatom_j}}{\lnorm{\noiseatom_j}} \ll 1 \quad \text{and} \quad \lnorm{\dictatom_j}^2 + \lnorm{\noiseatom_j}^2 = \lnorm{\extd\dictatom_j}^2 = \var[x_{i,j}] = 1,
	\end{equation}
	the value of $\lnorm{\dictatom_j}$ must be very small.
	Thus, we are actually in the same situation as before (the first case), leading to a larger value of $R$.
	Problems may therefore occur if some of the non-zero components of the representing feature vector $\tru\fv$ suffer from \eqref{eq:applications:standardization:badsnr}.
	This observation gives a further justification of Definition~\ref{def:results:dictionary:optimalrepr}, since it tries to avoid a low SNR by minimizing the noise part $\lnorm{\dictnoise \tru\fv}$.
%
\item
	For general coefficient sets $\sset$, the situation becomes more difficult.
	Even if $\sset$ is a polytope, it is not guaranteed anymore that the factor $\extd{D}_{\max}$ is bounded by the maximal column norm of $\extd\dict$.
	For that reason, different strategies may be required to gain control of the asymptotic behavior of $\effdim{\extd\dict\sset}$.
	Which type of preprocessing is appropriate clearly depends on the specific application, but as illustrated above, a careful analysis of the effective dimension could serve as a promising guideline here.
	\qedrmkhere
\end{rmklist}
\end{remark}

The previous considerations have merely focused on the first summand of the error bound \eqref{eq:results:selection:bound}.
However, in order to achieve a meaningful interpretation of Theorem~\ref{thm:results:selection}, it is equally 
important to bound the noise parameter $\advdev$, given by \eqref{eq:results:guarantee:noiseparam}. In this situation, the SNR plays a more significant role.\footnote{The SNR is defined according to \eqref{eq:results:dictionary:snr}, that is, $\SNR = 1 / \tru\stddev^2$.}
Since one has $\advdev = 0$ in the noiseless case (``perfect SNR''), we may hope that $\advdev$ is still small if the SNR is sufficiently high.
The following example verifies this intuition for the standard cases of linear and binary outputs. In the sequel, the noise variance $\tru\stddev^2 \ (= \SNR^{-1})$ should be regarded as a fixed (small) number, which is determined by the choice of $\tru\fv$.

\begin{example}\label{ex:applications:standardization:noiseparam}
\begin{rmklist}
\item
	\emph{Noisy linear observations.} Similarly to Example~\ref{ex:results:estimatorsignal:lossmodel}\ref{ex:results:estimatorsignal:lossmodel:linear}, let us assume that $\fobs(g) := g + \xi$ with additive mean-zero noise $\xi$. Recalling $\tru{s} = \sp{\lat_i}{\trusig}$, $\tru{n} = \sp{\latnoise_i}{\dictnoise \tru\fv}$ and \eqref{eq:results:guarantee:modddrepresentation}, the observations take the following form:
	\begin{equation}
		y_i = \tru{s} + \xi_i \quad \text{and} \quad y_i'' = \tfrac{\tru{s} + \tru{n}}{\sqrt{1+ \tru\stddev^2}} + \xi_i.
	\end{equation}
	Since $\tru{s} \distributed \Normdistr{0}{1}$ and $\tru{n} \distributed \Normdistr{0}{\tru\stddev^2}$ are independent, we have
	\begin{equation}
		Y_i := y_i - y_i'' = \tru{s} - \tfrac{\tru{s} + \tru{n}}{\sqrt{1+ \tru\stddev^2}} = \Big(1 - \tfrac{1}{\sqrt{1+ \tru\stddev^2}}\Big) \tru{s} + \tfrac{1}{\sqrt{1+ \tru\stddev^2}} \tru{n} \distributed \Normdistr{0}{\tru{\std{\stddev}}^2}
	\end{equation}
	with
	\begin{equation}
		\tru{\std{\stddev}}^2 := \Big(1 - \tfrac{1}{\sqrt{1+ \tru\stddev^2}}\Big)^2 + \tfrac{\tru\stddev^2}{1+ \tru\stddev^2} = 2 - \tfrac{2}{\sqrt{1+ \tru\stddev^2}}.
	\end{equation}
	Thus $\mean[Y_i^2] = \tru{\std{\stddev}}^2$, and by \cite[Prop.~7.5]{foucart2013cs}, we conclude that $Y_i^2 = (y_i - y_i'')^2$ is subexponential:
	\begin{equation}
		\prob[|Y_i^2 - \tru{\std{\stddev}}^2| \geq u] \leq \prob[\abs{\tru{\std{\stddev}}^{-1} Y_i} \geq \sqrt{\tru{\std{\stddev}}^{-2} u - 1}][\Big] \leq \exp(-\tfrac{\tru{\std{\stddev}}^{-2} u-1}{2}) \leq \sqrt{e} \cdot \exp(-\tfrac{u}{2\tru{\std{\stddev}}^2})
	\end{equation}
	for all $u \geq \tru{\std{\stddev}}^{2}$.
	Finally, we apply Bernstein's inequality (\cite[Cor.~7.32]{foucart2013cs} with $\beta = \sqrt{e}$, $\kappa = 1 / (2 \tru{\std{\stddev}}^2)$, and $t = 3 \tru{\std{\stddev}}^2 m$) to obtain a high-probability bound on the noise parameter:
	\begin{align}
		\prob[\advdev \leq 2 \tru{\std{\stddev}}] &= \prob[\tfrac{1}{m}\sum_{i = 1}^m Y_i^2  \leq 4 \tru{\std{\stddev}}^2][\Big] \geq \prob[\abs{\tfrac{1}{m}\sum_{i = 1}^m (Y_i^2 - \tru{\std{\stddev}}^2)} \leq 3 \tru{\std{\stddev}}^2][\Big] \geq 1 - 2 \exp\Big( - \tfrac{2m}{5} \Big).
	\end{align}
	Observing that $\tru{\std{\stddev}} = O(\tru{\stddev}) = O(\SNR^{-1/2})$ as $\tru{\stddev} \to 0$, such a bound could be now easily incorporated into the statement of Theorem~\ref{thm:results:selection}.
\item\label{ex:applications:standardization:noiseparam:binary}
	\emph{Binary observations.} We now consider $\fobs(g) := \sign(g)$, which implies
	\begin{equation}
		y_i = \sign(\tru{s}) \quad \text{and} \quad y_i'' = \sign\Big(\tfrac{\tru{s} + \tru{n}}{\sqrt{1+ \tru\stddev^2}}\Big) = \sign(\tru{s} + \tru{n}).
	\end{equation}
	These are Bernoulli variables, so that $\advdev^2$ actually measures the fraction of bit-flips (up to a factor of $4$) caused by the noise term $\tru{n}$.
	Hence, let us first compute the probability of a single bit-flip:
	\begin{align}
		\prob[y_i \neq y_i''] &= \prob[\tru{s} \cdot (\tru{s} + \tru{n}) < 0] = \iint_{s (s+n) < 0} \tfrac{1}{\sqrt{2\pi}} \exp(-\tfrac{s^2}{2}) \cdot \tfrac{1}{\sqrt{2\pi \tru\stddev^2}} \exp(-\tfrac{n^2}{2\tru\stddev^2}) \ dn \ ds \\
		&= 2 \cdot \tfrac{1}{2\pi \tru\stddev} \int_{0}^\infty \int_{-\infty}^{-s} \exp(-\tfrac{s^2}{2}) \cdot \exp(-\tfrac{n^2}{2\tru\stddev^2})\ dn \ ds \\
		&= \tfrac{1}{\pi \tru\stddev} \Big( \tfrac{\pi\tru\stddev}{2} - \tru\stddev \arctan(\tfrac{1}{\tru\stddev}) \Big) = \tfrac{1}{2} - \tfrac{1}{\pi} \arctan(\tfrac{1}{\tru\stddev}) =: \tru{p}.
	\end{align}
	Setting again $Y_i := y_i - y_i''$, we conclude that $\prob[Y_i^2 = 4] = 1 - \prob[Y_i^2 = 0] = \tru{p}$. In particular, $\mean[Y_i^2] = 4 \tru{p}$ and $\abs{Y_i^2 - \mean[Y_i^2]} \leq 4$. An application of Hoeffding's inequality (\cite[Thm.~7.20]{foucart2013cs} with $B_i = 4$ and $t = 5 \tru{p} m$) finally yields
	\begin{align}
		\prob[\advdev \leq 3\sqrt{\tru{p}}] &= \prob[\tfrac{1}{m}\sum_{i = 1}^m Y_i^2  \leq 9 \tru{p}][\Big] = \prob[{\tfrac{1}{m}\sum_{i = 1}^m (Y_i^2 - \mean[Y_i^2])  \leq 5\tru{p}}][\Big] \geq 1 - \exp \Big(- \tfrac{25\tru{p}^2 m}{32} \Big).
	\end{align}
	This bound is also relevant to Theorem~\ref{thm:results:selection}, at least for small (but fixed) values of $\tru{\stddev}$, since we have $\sqrt{\tru{p}} = O(\sqrt{\tru{\stddev}}) = O(\SNR^{-1/4})$ as $\tru{\stddev} \to 0$.
	\qedrmkhere
\end{rmklist}
\end{example}

\subsection{Sparse Feature Extraction from MS Data}
\label{subsec:applications:msdata}

Throughout this paper, the challenge of (sparse) feature extraction from MS~data has been a driving motivation for our modeling. With the full abstract framework of Section~\ref{sec:results} available, we are now ready to formulate a rigorous recovery guarantee for this specific application.
In order to make the specifications of \ref{assump:model:obsmodelrefined} and \ref{assump:model:datamodel} precise, let us first recall the vague model description from Subsection~\ref{subsec:intro:motivation}.
Each sample $i \in \{1, \dots, m\}$ corresponds to a patient who suffers from a certain disease ($y_i = +1$) or not ($y_i = -1$).
The associated signal variables $\lat_i = (s_{i,1}, \dots, s_{i,p}) \distributed \Normdistr{\vnull}{\Covmatr}$ represent the (centered) concentrations of the examined protein collection. In other words, each index $k \in \{ 1, \dots, p\}$ of $\lat_i$ stands for a particular type of protein and the corresponding entry contains its molecular concentration for the $i$-th patient.
Our hope is that the disease labels can be accurately predicted by a \emph{sparse $1$-bit observation scheme} ($\fobs = \sign$)
\begin{equation}\label{eq:applications:msdata:1bitmodel}
	y_i^0 := \sign(\sp{\lat_i}{\trusig}), \quad i = 1, \dots, m,
\end{equation}
where the signal vector $\trusig \in \R^p$ is assumed to be \emph{$s$-sparse}, that is, $\lnorm{\trusig}[0] \leq s$.
However, the true disease labels $y_1, \dots, y_m \in \{-1,+1\}$ may ``slightly'' differ from this very simple model.
As proposed by \ref{assump:model:obsmodelrefined}, we shall measure this mismatch by
\begin{equation}
	\tru\advdev := \sqrt{\tfrac{1}{m} \sum_{i = 1}^m \abs{y_i^0 - y_i}^2},
\end{equation}
which is (up to a factor of $2$) equal to the root of the \emph{fraction of wrong observations}.

The actual input data $\data_1, \dots, \data_m \in \R^d$ is generated via \emph{mass spectrometry} (\emph{MS}); see also Subsection~\ref{subsec:intro:motivation}, especially Figure~\ref{fig:intro:motivation:msdata}.
To give a formal definition of a \define{mass spectrum} according to \ref{assump:model:datamodel}, we follow the basic approach of \cite[Sec.~4.1]{genzel2015master}, which describes 
the raw data as a linear combination of \define{Gaussian-shaped peaks} plus low-amplitude \define{baseline noise}. In this situation, it is also helpful to be aware of the physical meaning of the data variables $\data_i = (x_{i,1}, \dots, x_{i,d})$: Each index $j \in \{1, \dots, d\}$ corresponds to a specific mass value (which is proportional to $j$), while the associated \define{mass channel} of the MS machine detects (counts) all those molecules having this specific mass value.
Due to the underlying physics, this process is not exact so that the abundance of each protein, i.e., the total count of detected molecules, typically spreads over several channels.
Consequently, one eventually observes wide peaks instead of sharp spikes (cf. Figure~\ref{fig:intro:motivation:msdata}).
To keep our illustration as simple as possible, we assume that the (raw) feature atoms\footnote{The word ``raw'' indicates that we do not care about standardizing the data for now. A standardization is then considered as a separate step below.} are discrete samplings of Gaussian-shaped functions
\begin{equation}
	\atom_k' = (a_{k,1}', \dots, a_{k,d}') := (G_{I_k,c_k,d_k}(1), \dots, G_{I_k,c_k,d_k}(d)) \in \R^d, \quad k = 1, \dots, p,
\end{equation}
with
\begin{equation}
	G_{I_k, c_k, d_k}(t) := I_k \cdot \exp\Big( -\tfrac{(t - c_k)^2}{d_k^2} \Big), \quad t \in \R, \quad I_k \geq 0, \ c_k \in \R, \ d_k > 0.
\end{equation}
Thus, $I_k$ determines the height\footnote{Note that $I_k$ is still a fixed number. By \ref{assump:model:datamodel}, the maximal height of the $k$-peak is actually given by $s_{i,k} \cdot I_k$, which also depends on the considered sample $i$.} (intensity) of the $k$-th peak, $d_k$ its width, and $c_k$ its mass center.

The (raw) noise atoms, on the other hand, simply capture the independent baseline noise that is present in each mass channel:
\begin{equation}
	\atomnoise_l' := \stddev_l \vunit_l \in \R^d, \quad l = 1, \dots, q := d,
\end{equation}
where $\stddev_l \geq 0$ and $\vunit_l$ is the $l$-th unit vector of $\R^d$.

If we would just use $\atom_1', \dots, \atom_p', \atomnoise_1', \dots, \atomnoise_d'$ as input parameters for \ref{assump:model:datamodel}, the issue of Remark~\ref{rmk:applications:standardization}\ref{rmk:applications:standardization:snr} would become relevant, since the maximal intensities $I_1, \dots, I_p$ may be of different orders of magnitude in practice. In fact, there are usually peaks of very low height, whereas some others are huge in magnitude (see again Figure~\ref{fig:intro:motivation:msdata}). However, this does by far not imply that the latter ones are more important than the others.
Raw MS~data is therefore often standardized before feature selection.
For this purpose, let
\begin{align}
	\dict' &= \matr{\dictatom_1' \mid \dots \mid \dictatom_d'} := \matr{\atom_1' \mid \dots \mid \atom_p'}^\T \in \R^{p \times d}, \\
	\dictnoise' &= \matr{\dictnoise_1' \mid \dots \mid \dictnoise_d'} := \matr{\atomnoise_1' \mid \dots \mid \atomnoise_d'}^\T = \diag{\stddev_1, \dots, \stddev_d} \in \R^{d\times d}.
\end{align}
denote the raw dictionaries. According to \eqref{eq:applications:standardization:variance}, we compute the variance of the raw feature variables
\begin{equation}
	(\stddev_j')^2 := \var[x_{i,j}'] = \lnorm{\dictatom_j'}^2 + \lnorm{\noiseatom_j'}^2 = \lnorm{\dictatom_j'}^2 + \stddev_j^2, \quad j = 1, \dots, d,
\end{equation}
and standardize the dictionary atoms by
\begin{equation}\label{eq:applications:msdata:atomstd}
	\dictatom_j := \dictatom_j' / \stddev_j' \quad \text{and} \quad \noiseatom_j := \noiseatom_j' / \stddev_j', \quad j = 1, \dots, d.
\end{equation}
The resulting data model \ref{assump:model:datamodel} is then based on the extended dictionary
\begin{equation}\label{eq:applications:msdata:extdatomstd}
	\extd\dict = \matr{\extd\dictatom_1 \mid \dots \mid \extd\dictatom_d} := \left[\begin{array}{c|c|c}\dictatom_1 & \dots & \dictatom_d \\ \noiseatom_1 & \dots & \noiseatom_d\end{array}\right] = \matr{\dict \\ \dictnoise} = \matr{\atoms^\T \\ \atomsnoise^\T} \in \R^{(p+d) \times d}.
\end{equation}

Using these model assumptions, we can now state a corollary of Theorem~\ref{thm:results:selection:refined} for variable selection from MS~data.
This can be also regarded as an extension of Theorem~\ref{thm:intro:sparseselection} to the noisy case.
\begin{corollary}\label{cor:applications:msdata}
	Suppose that the standardized input data $\{(\data_i, y_i)\}_{1 \leq i \leq m}$ follow \ref{assump:model:obsmodelrefined} and \ref{assump:model:datamodel} with the above specifications. Furthermore, let \ref{assump:estimator:regularity} and \ref{assump:estimator:convexity} hold true. We assume that $\lnorm{\sqrt{\Covmatr}\trusig} = 1$ and $\trusig \in R\dict\ball[1][d]$ for some $R > 0$.
	Then there exists a constant $C > 0$, depending on the used loss function $\loss$, such that the following holds with high probability:
	
	If the number of samples obeys
	\begin{equation}
		m \geq C^4 \cdot R^2 \cdot \log(2d),
	\end{equation}
	then, defining $\solu\sig := \dict \solu\fv$ for any minimizer $\solu\fv$ of \eqref{eq:results:estimatorsignal:estimator}, we have
	\begin{equation}
		\lnorm{\sqrt{\Covmatr}(\solu\sig - \tru\snrscal\trusig)} = \lnorm{\sqrt{\Covmatr}(\dict\solu\fv - \tru\snrscal \dict\tru\fv)} \leq C \bigg( \Big(\frac{R^2 \cdot \log(2d)}{m}\Big)^{1/4} + \sqrt{\tru{\stddev}} + \tru\advdev \bigg),
	\end{equation}
	where $\tru\snrscal := \sqrt{\frac{2}{ \pi (1 + \tru\stddev^2)}}$.
\end{corollary}
\begin{proof}
First, we observe that $\scalfac = \sqrt{2 / \pi} < 1$ (by Example~\ref{ex:results:estimatorsignal:lossmodel}\ref{ex:results:estimatorsignal:lossmodel:binary}), and therefore $\scalfac \trusig \in \scalfac R\dict\ball[1][d] \subset R\dict\ball[1][d]$.
Applying the results from Subsection~\ref{subsec:applications:standardization} (in particular, the bound \eqref{eq:applications:standardization:effdimboundstd}, Remark~\ref{rmk:applications:standardization}\ref{rmk:applications:standardization:anisotropic}, and Example~\ref{ex:applications:standardization:noiseparam}\ref{ex:applications:standardization:noiseparam:binary}), the claim now immediately follows from Theorem~\ref{thm:results:selection:refined}.
\end{proof}

The significance of Corollary~\ref{cor:applications:msdata} primarily depends on the scaling parameter $R > 0$ and on the noise variance $\tru{\stddev}^2$.
To investigate this issue further, we need to better understand the structure of the (standardized) dictionary $\extd\dict$.
A key observation of \cite[Subsec.~4.3.3]{genzel2015master} was that, mostly, the peak centers are ``well-separated,'' i.e.,
\begin{equation}\label{eq:applications:msdata:wellseppeaks}
	\abs{c_k - c_{k'}} \gg \max\{ d_k, d_{k'} \} \quad \text{for all $k, k' = 1, \dots, p$ with $k \neq k'$.}
\end{equation}
In other words, the overlap between two different peaks is usually very small (cf. Figure~\ref{fig:intro:motivation:msdata}).
Mathematically seen, this assumption implies that the feature atoms are almost orthogonal, that is, $\sp{\atom_k}{\atom_{k'}} \approx 0$ for $k \neq k'$.
Figure~\ref{fig:applications:msdata:dict} shows a typical example of $\extd\dict$ in such a situation.
\begin{figure}
	\centering
	\includegraphics[width=.4\textwidth]{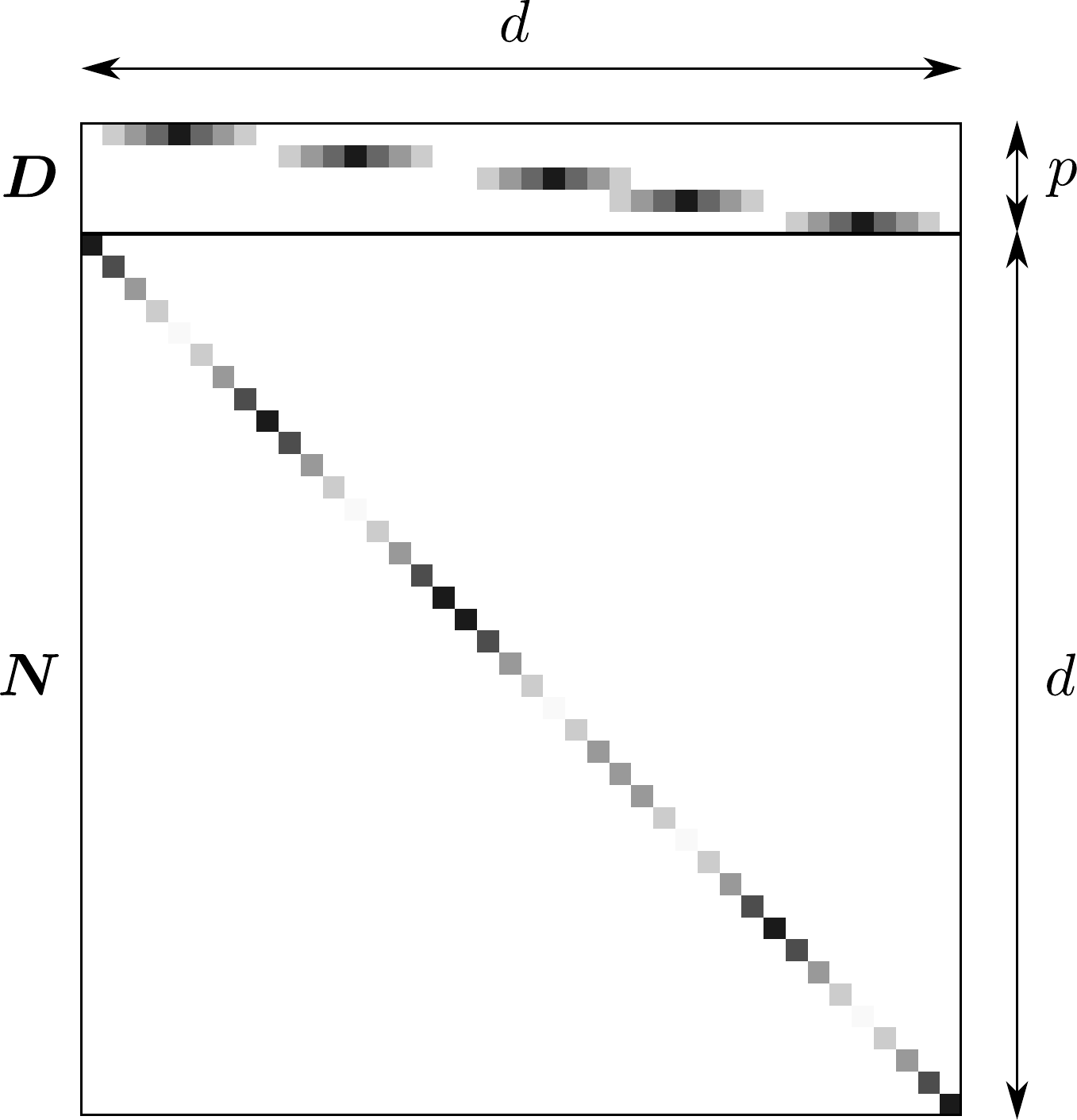}
	\caption{Visualization of $\extd\dict \in \R^{(p+d)\times d}$ in the case of almost isolated peaks (the grayscale values correspond to the magnitudes of the entries, $\text{white} = 0$ and $\text{black} = 1$). Indeed, the $p = 5$ feature atoms do only virtually overlap, so that the noise atoms precisely dominate those feature variables which are ``far away'' from any of the peak centers.}
	\label{fig:applications:msdata:dict}
\end{figure}
This visualization suggests to distinguish between two extreme cases:\footnote{In realistic scenarios, the dimension $d$ is usually much larger than the number of peaks $p$. Therefore, one of these two ``extreme'' cases applies to most of the feature variables.}
\begin{enumerate}[label=(\arabic*)]
\item\label{case:applications:msdata:dictfeat}
	\emph{Features contained in the essential support of a peak.}
	Let us assume that the $j$-th variable is very close to the center $c_k$ of the $k$-th peak ($k \in \{1, \dots, p\}$).
	Then, we have $\lnorm{\dictatom_j'}^2 \gg \stddev_j^2$, implying $(\stddev_j')^2 = \lnorm{\dictatom_j'}^2 + \stddev_j^2 \approx \lnorm{\dictatom_j'}^2$.
	Due to \eqref{eq:applications:msdata:atomstd} and \eqref{eq:applications:msdata:extdatomstd}, the associated dictionary atom now approximately takes the form
	\begin{equation}\label{eq:applications:msdata:dictfeat}
		\extd\dictatom_j = \matr{\dictatom_j \\ \noiseatom_j} \approx \matr{\vunit_k \\ \vnull} \in \R^{p+d},
	\end{equation}
	where $\vunit_k \in \R^p$ is the $k$-th unit vector of $\R^p$. Thus, the contribution of the noise atoms (baseline noise) is negligible here.
\item\label{case:applications:msdata:dictnoise}
	\emph{Features far away from any of the peak centers.} Now, $\lnorm{\dictatom_j'}^2 \ll \stddev_j^2$, which allows for an approximation $(\stddev_j')^2 = \lnorm{\dictatom_j'}^2 + \stddev_j^2 \approx \stddev_j^2$. For that reason, the dictionary atom is dominated by the noise part:
	\begin{equation}
		\extd\dictatom_j = \matr{\dictatom_j \\ \noiseatom_j} \approx \matr{\vnull \\ \vunit_j} \in \R^{p+d},
	\end{equation}
	where $\vunit_j \in \R^d$ is the $j$-th unit vector of $\R^d$.
	According to \eqref{eq:applications:standardization:badsnr}, these variables suffer from a very low SNR.
\end{enumerate}
With regard to our main goal to achieve a small variance $\tru\stddev^2 = \lnorm{\dictnoise \tru\fv}$, we can conclude that the support of an optimal representation $\tru\fv$ is (most likely) contained within those variables corresponding to Case~\ref{case:applications:msdata:dictfeat}.
This particularly shows that the issue of Remark~\ref{rmk:applications:standardization}\ref{rmk:applications:standardization:snr} can be easily resolved by standardizing the raw MS~data.

Restricting to the features of Case~\ref{case:applications:msdata:dictfeat}---and these are actually the only ones of interest---the dictionary $\dict$ enjoys a very simple structure (while $\dictnoise$ is negligible anyway). In fact, it basically consists of the unit vectors $\vunit_1, \dots, \vunit_p \in \R^p$, even though some of them may repeatedly occur, due to the spatial extent of the single peaks.
The error estimate of Corollary~\ref{cor:applications:msdata} therefore implies that the non-zero entries of the estimated feature vector $\solu\fv$ do also correspond to Case~\ref{case:applications:msdata:dictfeat}.\footnote{This conclusion is a bit vaguely formulated. For a more precise argument, one needs to observe that the error bound of Corollary~\ref{cor:applications:msdata} holds for $\lnorm{\dictnoise\solu\fv - \tru\snrscal \dictnoise\tru\fv}$ as well (see Remark~\ref{rmk:proofs:selection}\ref{rmk:proofs:selection:goodsnr}), which implies that $\dictnoise\solu\fv \approx \tru\snrscal \dictnoise\tru\fv \approx \vnull$.}
In particular, it is possible to assign each non-zero component of $\solu\fv$ to a specific peak.
A physician could now use his/her medical expertise to identify these peaks (feature atoms) with the underlying proteins (signal variables).
In that sense, the unknown transformation $\extd\dict$ can be invoked ``by hand,'' and consequently, we can even give a positive answer to \ref{question:intro:practical} in this situation.
For a numerical verification of the above argument, the interested reader is referred to the works of \cite[Chap.~5]{genzel2015master} and \cite{conrad2015spa}.

The natural (one-to-one) correspondence between the variables of the signal and data domain also allows us to specify a good choice of $R$: Since $\dict$ behaves similarly to the identity matrix, we may assume that there exists an (optimal) representation $\tru\fv$ of $\trusig$ which is $s$-sparse by itself and approximately preserves the norm, i.e., $\lnorm{\trusig} = \lnorm{\dict\tru\fv} \approx \lnorm{\tru\fv}$.
By the Cauchy-Schwarz inequality, we then obtain
\begin{equation}
	\lnorm{\tru\fv}[1] \leq \sqrt{\lnorm{\tru\fv}[0]} \cdot \lnorm{\tru\fv} \approx \sqrt{s} \cdot \lnorm{\trusig} \leq \sqrt{s} \cdot \norm{\sqrt{\Covmatr}^{-1}} \cdot \lnorm{\sqrt{\Covmatr}\trusig} = \sqrt{s} \cdot \norm{\sqrt{\Covmatr}^{-1}}.
\end{equation}
Hence, the hypothesis of Corollary~\ref{cor:applications:msdata} can be satisfied with $R \approx \sqrt{s} \cdot \norm{\sqrt{\Covmatr}^{-1}}$, and the number of required observations is essentially dominated by the degree of sparsity:
\begin{equation}
m \gtrsim C \cdot \norm{\sqrt{\Covmatr}^{-1}}^2 \cdot s \cdot \log(2d).
\end{equation}
This indicates that feature extraction from MS~data is already possible with a relatively few samples, supposed that the disease fingerprint $\trusig$ is sufficiently sparse and the covariance matrix $\Covmatr$ is not too ill-conditioned.
Such a conclusion is especially appealing for practical scenarios, since many clinical studies only encompass a very small sample set .
Moreover, it is worth mentioning that---regarding the resulting error bound of Corollary~\ref{cor:applications:msdata}---the estimator \eqref{eq:results:estimatorsignal:estimator} enjoys a certain type of \emph{oracle property} in the sense that it behaves almost as good as if we would work directly in the signal domain.

\begin{remark}
\begin{rmklist}
\item
	The discussion of Corollary~\ref{cor:applications:msdata} is based on some heuristic assumptions which do not always strictly hold in practice.
	For example, we may run into trouble if a peak is ``buried'' in the baseline noise, i.e., $I_k \ll \stddev_l$ when $l$ is close to $c_k$. In this case, \emph{all} associated feature variables would suffer from a low SNR (in the sense of \eqref{eq:applications:standardization:badsnr}) and the approximation of Case~\ref{case:applications:msdata:dictfeat} becomes very rough.
	Another undesirable situation arises if the \emph{sparse} $1$-bit model of \eqref{eq:applications:msdata:1bitmodel} is too inaccurate. Then, we would loose any control of the mismatch parameter $\tru\advdev$ and the error estimate of Corollary~\ref{cor:applications:msdata} becomes meaningless.
\item
	The bound of Corollary~\ref{cor:applications:msdata} could be substantially improved if some prior information about the data is available.
	For instance, by using statistical tests or medical expertise, one might already have a rough idea of what peaks (proteins) are more important than others.
	Such an additional knowledge could be easily incorporated via a \define{weighted standardization} as proposed in \cite{genzel2015master}.
	Alternatively, the recent concept of \define{weighted sparsity} could be beneficial (see \cite{rauhut2015weighted} for example), which would basically involve an \emph{anisotropic} rescaling of the coefficient set $\sset = R\ball[1][d]$. \qedrmkhere
\end{rmklist}
\end{remark}

\subsection{Practical Guidelines and Other Potential Applications}
\label{subsec:applications:relevance}

In the previous subsection, we have seen how our framework can be used to establish theoretical guarantees for feature selection from MS~data.
But what about the general practical scope of our results?
First, one might wonder about the feasibility of the model assumptions made in Section~\ref{sec:results}.
The linear model of \ref{assump:model:datamodel} is in fact quite generic, since every (Gaussian) data can be factorized by \eqref{eq:results:model:datamodel} if the feature and noise atoms are appropriately chosen.
The same is true for the observation scheme of \ref{assump:model:obsmodelrefined} where we even permit adversarial (deterministic) noise in the output variables.
A stronger limitation is of course the hypothesis of \emph{Gaussian} signal factors. It is indeed very unlikely that $\lat_i$ and $\latnoise_i$ precisely follow a Gaussian distribution in realistic situations. Hence, an emerging goal for future work should be to allow for more complicated measurement designs, such as sub-Gaussians (see also Section~\ref{sec:conclusion}).

Next, let us return to our initial challenge of \ref{question:intro:practical}. When applying Theorem~\ref{thm:results:selection} to a specific problem, we are always faced with the drawback that the estimator \eqref{eq:results:estimatorsignal:estimator} yields a vector $\solu\fv$ in the data domain while the actual approximation statement holds in the signal domain for $\solu\sig = \dict \solu\fv$.
The practical significance of our results therefore heavily depends on the intrinsic dictionary $\dict$, which could be very different for each application.
Unfortunately, one cannot expect a general conclusion here, since $\dict$ can be arbitrary and might be (entirely) unknown.
However, we would like to point out some ``good'' scenarios for which a meaningful answer to \ref{question:intro:practical} can be given:
\begin{listing}
\item
	\emph{Well-separated feature atoms.}
	If the supports of $\atom_1, \dots, \atom_p$ do only slightly overlap, we are essentially in the desirable setting of Subsection~\ref{subsec:applications:msdata}, where $\dict$ consisted of almost isolated peak atoms (see Figure~\ref{fig:applications:msdata:dict}).
	Thus, we may assume again that the support of the representing vector $\tru\fv$ is associated with feature variables of the form \eqref{eq:applications:msdata:dictfeat}.
	Although the dictionary $\dict$ might be still highly redundant, its structure is rather simple in this case, and a domain expert should be able to identify the non-zero components of $\solu\fv$ with the underlying signal variables.
	
	But note that this strategy can quickly fail if the feature atoms superpose. In the setup of MS~data for example, one could think of several narrow peaks sitting on the top of a wider peak.
	The behavior of $\dict$ becomes much more complicated in such a situation, so that some prior knowledge is probably required to cope with variable selection (in the signal domain).
\item
	\emph{$\dict$ is approximately known.} Suppose that an approximate dictionary $\dict' \in \R^{p \times d}$ is available, e.g., as an outcome of \emph{dictionary learning} or a \emph{factor analysis}.
	Then we can define an estimator by $\solu\sig' := \dict'\solu\fv$. If the error bound of Theorem~\ref{thm:results:selection} holds true, the triangle inequality now implies
	\begin{equation}
		\lnorm{\solu\sig' - \tru\snrscal\trusig} \leq \lnorm{\solu\sig' - \solu\sig} + \lnorm{\solu\sig - \tru\snrscal\trusig} \leq \norm{\dict' - \dict} \cdot \lnorm{\solu\fv} + \modeldevconst \bigg( \Big(\frac{\effdim{\extd\dict\sset}}{m}\Big)^{1/4} + \advdev \bigg).
	\end{equation}
	Hence, if $\dict'$ and $\dict$ are sufficiently close, we can conclude that $\solu\sig'$ serves as a reliable substitute for the (unknown) estimator $\solu\sig$.
\item
	\emph{Output prediction.}
	The primal focus of many applications is rather on accurately predicting the output variable, e.g., the health status of a patient.
	In this case, it is not even necessary to work within the signal domain: Let $(\data, y) \in \R^d \times \Y$ be a data pair with $y$ unknown. Then, one can simply define $\solu{y} := \fobs(\sp{\data}{\solu\fv})$ as an estimator of $y$. Here, the function $\fobs$, or at least an approximate version of it, is of course assumed to be given.
\end{listing}

These guidelines might be helpful to identify more potential applications of our framework.
We strongly believe that a similar statement than Corollary~\ref{cor:applications:msdata} can be also shown for other types of real-world data, for example, \emph{microarrays}, \emph{neuroimages}, or \emph{hyperspectral images}. 
A major difficulty when establishing a novel result for feature extraction is to find an appropriate parameterization of the data and observation model.
Initially, it might be unclear how to split the underlying collection of signal and noise factors, i.e., to figure out which variables are relevant to \ref{assump:model:obsmodel} and which are probably not.
But even then, there is usually still a certain degree of freedom to specify the atoms of \ref{assump:model:datamodel}, and we have seen that the quality of our error bounds is quite sensitive to this choice.
For that reason, each specific application will require a very careful individual treatment, such as we did in Subsection~\ref{subsec:applications:msdata} for MS~data.

\section{Conclusion and Outlook}
\label{sec:conclusion}

Regarding the initial challenges of \ref{question:intro:theoretical} and \ref{question:intro:setup}, we can conclude that our main results, Theorem~\ref{thm:results:selection} and Theorem~\ref{thm:results:selection:refined}, provide fairly general solutions.
The key idea was to construct an (optimal) representation of the signal vector $\trusig$ in order to mimic the original problem of variable selection within 
the data domain. In this way, we were able to prove recovery statements for standard estimators, which only take the raw samples $\{(\data_i, y_i)\}_{1 \leq i \leq m}$ 
as input. Interestingly, it has turned out that this approach works almost as good as if one would explicitly know the hidden signal factors $\lat_i$ of the data (``oracle property'').
Another remarkable observation was that the used loss function $\loss$ as well as the (unknown) non-linearity $\fobs$ do only have a minor impact on the qualitative behavior of the error bounds.

In Section~\ref{sec:applications}, we have also discussed the practical scope of our findings, in particular the issue of \ref{question:intro:practical}.
As the setting of Theorem~\ref{thm:results:selection} and Theorem~\ref{thm:results:selection:refined} is quite generic, a practice-oriented statement can be usually only drawn up if some further assumptions on the data model \ref{assump:model:datamodel} are made.
This was illustrated for the example of MS~data, but there should be many other applications for which similar guarantees can be proven.
For that matter, we hope that our results could at least serve as a promising indicator of successful feature extraction from real-world data.

Finally, let us sketch several improvements and extensions of our framework which could be interesting to study in future research:
\begin{listing}
\item
	\emph{Non-Gaussian observations.}
	In the course of Subsection~\ref{subsec:applications:relevance}, we have already pointed out  that the assumption of Gaussian signal variables is often too restrictive in realistic scenarios.
	Therefore, it would be essential to allow for more general distributions, e.g., sub-Gaussians or even heavily-tailed random variables.
	The most difficult step towards such an extension is to handle the non-linearity $\fobs$ in a unified way.
	In the noisy linear case of Example~\ref{ex:results:estimatorsignal:lossmodel}\ref{ex:results:estimatorsignal:lossmodel:linear}, our results may be easily extended to sub-Gaussian factors, using techniques from \cite{tropp2014convex,liaw2016randommat,mendelson2007subgaussian}.
	If $\fobs$ produces binary outputs on the other hand, the situation becomes much more complicated.
	There exists in fact a simple counterexample based on Bernoulli variables for which signal recovery is impossible (see \cite[Rem.~1.5]{plan2013onebit}).
	An interesting approach has been given in \cite{ai2014onebitsubgauss}, excluding such extreme cases, but the authors just consider a linear loss for their estimator, which does not fit into our setup.
	However, we strongly believe that similar arguments can be used to adapt the statement of Theorem~\ref{thm:results:selection} to (a sub-class of) sub-Gaussian distributions.
\item
	\emph{More advanced sample models.}
	In order to achieve a smaller adversarial noise parameter $\tru\advdev$ in \ref{assump:model:obsmodelrefined}, it might be useful to go beyond a single-index model. For example, the output variables could be described by a \define{multi-index observation rule} of the form
	\begin{equation}
		y_i^0 = \fobs(\sp{\lat_i}{\trusig^{(1)}}, \sp{\lat_i}{\trusig^{(2)}}, \dots, \sp{\lat_i}{\trusig^{(D)}}), \quad i = 1, \dots, m,
	\end{equation}
	where $\trusig^{(1)}, \dots, \trusig^{(D)} \in \R^p$ are unknown signal vectors and $\fobs \colon \R^D \to \Y$ is again a scalar function.
	Apart from that, one could also think of a certain non-linearity in the linear factor model \ref{assump:model:datamodel}, especially in the noise term.
\item
	\emph{Incorporation of prior knowledge.}
	A crucial feature of our main results is the fact that they impose only relatively few assumptions on the data model (the atoms $\atom_1, \dots, \atom_p$ and $\atomnoise_1, \dots, \atomnoise_q$ can be arbitrary).
	In most applications however, the input data obey additional structural constraints, such as the characteristic peak patterns of mass spectra.
	On this account, we have sketched various rules-of-thumb in Section~\ref{sec:applications}, but a rigorous treatment of these heuristics remains largely unexplored.
	For example, one could ask for the following:
	Can we simultaneously learn the signal vector $\trusig$ and the feature dictionary $\dict$?
	How to optimally preprocess (reweight) the data if some feature variables are known to be more important than others?
	And to what extent does this improve our error bounds?
\item
	\emph{Optimality of the representation.}
	In Definition~\ref{def:results:dictionary:optimalrepr}, we have introduced a representative feature vector $\tru\fv$ that is optimal with respect to the SNR (cf. \eqref{eq:results:dictionary:snr}).
	This is indeed a very natural way to match the true observation $y_i$ with its perturbed counterpart $y_i''$ from \eqref{eq:results:guarantee:modddrepresentation}, but the impact of the non-linearity $\fobs$ is actually disregarded.
	The error bound of Theorem~\ref{thm:results:selection}, in contrast, involves the noise parameter $\advdev$, which also depends on $\fobs$.
	For this reason, it is not clear whether our notion of optimality always leads to the best possible estimates.
	A more elaborate approach would be to choose $\tru\fv$ such that (the expected value of) $\advdev$ is minimized, which would in turn require a very careful study of $\fobs$. But fortunately, Example~\ref{ex:applications:standardization:noiseparam} indicates that we would end up with the same result in most situations anyway.
\item
	\emph{General convex loss functions.} It can be easily seen that some popular convex loss functions do not fit into our setup.
	A prominent example is the \define{hinge loss} (used for \define{support vector machines}), which is neither strictly convex nor differentiable.
	However, the recent work of \cite{kolleck2015l1svm} has shown that signal recovery is still possible for this choice, even though the authors consider a more restrictive setting than we do.
	Another interesting issue would concern the design of adaptive loss functions: Suppose the non-linearity $\fobs$ of the observation model is (approximately) known, can we construct a good (or even optimal) loss $\loss$?
\end{listing}

\section{Proofs}
\label{sec:proofs}

\subsection{Proof of Theorem \ref{thm:intro:sparseselection}}
\label{subsec:proofs:sparseselection}

The claim of Theorem~\ref{thm:intro:sparseselection} follows from a straightforward application of our main result, Theorem~\ref{thm:results:selection}, which is proven in the next subsection.
\begin{proof}[Proof of Theorem \ref{thm:intro:sparseselection}]
We first observe that the data $\data_i = \atoms \lat_i$ does not involve any noise variables, so that one can simply assume $q = 0$ and $\extd\dict = \dict$.
Since $\fobs = \sign$ and $\loss = \losssq$, Example~\ref{ex:results:estimatorsignal:lossmodel}\ref{ex:results:estimatorsignal:lossmodel:binary} yields $\scalfac = \sqrt{2 / \pi} < 1$.
Thus, choosing $\sset := R \ball[1][d]$, we have
\begin{equation}
\scalfac \trusig \in \scalfac R \dict \ball[1][d] \subset R \dict \ball[1][d] = \dict\sset.
\end{equation}
The noise term $\tru{n}$ is equal to zero by assumption, which means that the feature vector $\tru\fv$ is already an optimal representation of $\trusig$.
In particular, we conclude that $\tru\stddev^2 = 0$, implying $y_i'' = y_i$ for all $i = 1, \dots, m$ and also $\advdev = 0$.
An application of \eqref{eq:results:estimatorsignal:effdim:dictionary:dictpolytope} finally gives the bound
\begin{equation}
	\effdim{\extd\dict\sset} = R^2 \cdot \effdim{\dict\ball[1][d]} \lesssim R^2 \cdot \dictenergy^2  \cdot \log(2d),
\end{equation}
and the claim is now established by the statement of Theorem~\ref{thm:results:selection}.
\end{proof}

\subsection{Proofs of Theorem \ref{thm:results:selection} and Theorem \ref{thm:results:selection:refined}}
\label{subsec:proofs:selection}

We have already pointed out that we would like to apply the abstract framework of \cite{genzel2016estimation} to prove Theorem~\ref{thm:results:selection} and Theorem~\ref{thm:results:selection:refined}.
In order to keep our exposition self-contained, let us first restate a deep result from \cite{genzel2016estimation}, where the notation is adapted to the setup of this work:
\begin{theorem}[\protect{\cite[Thm.~1.3]{genzel2016estimation}}]\label{thm:proofs:highdimestimation}
Let $\tru\v \in \S^{n-1}$ be a unit-norm vector. Set
\begin{equation}\label{eq:proofs:highdimestimation:obsmodel}
	\tilde{y}_i := \fobs(\sp{\meas_i}{\tru\v}), \quad i = 1, \dots,m,
\end{equation}
for i.i.d.\ Gaussian samples $\meas_1, \dots, \meas_m \distributed \Normdistr{\vnull}{\Idm{n}}$ and a (random) function $\fobs\colon \R \to \Y$.
We assume that \ref{assump:estimator:regularity}, \ref{assump:estimator:convexity} are fulfilled and that $\scalfac$, $\modeldev$, $\modeldeveta$ are defined according to \eqref{eq:results:estimatorsignal:modelparam}.
Moreover, suppose that $\scalfac \tru\v \in L$ for a bounded, convex subset $L \subset \R^n$.
Then there exists a constant of the form $\modeldevconst' = C' \cdot \max\{1, \modeldev, \modeldeveta\} > 0$ with $C' > 0$ such that the following holds with high probability:\footnote{As in Theorem~\ref{thm:results:selection}, the constant $C'$ may depend on the ``probability of success'' as well as on the RSC-constant of $\loss$.}

If the number of observations obeys
\begin{equation}
	m \geq C' \cdot \effdim{L - \scalfac\tru\v},
\end{equation}
then any minimizer $\solu\v$ of
\begin{equation}\label{eq:proofs:highdimestimation:estimator}
	\min_{\v \in \R^n} \tfrac{1}{m} \sum_{i = 1}^m \loss(\sp{\meas_i}{\v}, y_i) \quad \text{subject to $\v \in L$} ,
\end{equation}
with arbitrary inputs $y_1, \dots, y_m \in \Y$, satisfies
\begin{equation}\label{eq:proofs:highdimestimation:bound}
	\lnorm{\solu\v - \scalfac\tru\v} \leq \modeldevconst' \bigg( \Big(\frac{\effdim{L - \scalfac\tru\v}}{m}\Big)^{1/4} + \Big(\tfrac{1}{m} \sum_{i = 1}^m \abs{\tilde{y}_i - y_i}^2\Big)^{1/2} \bigg).
	\end{equation}
\end{theorem}

The proof of our main result now follows from a sophisticated application of Theorem~\ref{thm:proofs:highdimestimation}:

\begin{proof}[Proof of Theorem \ref{thm:results:selection}]
Let us recall the approach of Remark~\ref{rmk:results:dictionary}\ref{rmk:results:dictionary:extddict} and work in the extended signal domain $\R^{p+q}$.
That means, we would like to apply Theorem~\ref{thm:proofs:highdimestimation} in a setup with $n := p + q$ and
\begin{equation}
	\meas_i := \extd{\lat}_i = (\lat_i, \latnoise_i) \distributed \Normdistr{\vnull}{\Idm{p+q}}, \quad i = 1, \dots, m.
\end{equation}
Next, we need to specify the vector $\tru\v$.
For this purpose, consider again
\begin{equation}
	 \tru{\extd\sig}' = \extd\dict \tru\fv = \matr{\dict \tru\fv \\ \dictnoise \tru\fv} = \matr{\trusig \\ \dictnoise \tru\fv},
\end{equation}
where the assumption $\scalfac \trusig \in \dict \sset$ ensures that an optimal representation $\tru\fv$ actually exists (cf. \eqref{eq:results:dictionary:fvstructured}).
Due to $\lnorm{\trusig} = 1$, we have $\tru\scalsig := \lnorm{\tru{\extd\sig}'} = \sqrt{1 + \lnorm{\dictnoise \tru\fv}^2} = \sqrt{1 + \tru\stddev^2}$, so that
\begin{equation}
	\tru\v := \tru\scalsig^{-1} \tru{\extd\sig}' = \tru\scalsig^{-1} \matr{\trusig \\ \dictnoise \tru\fv}
\end{equation}
satisfies $\lnorm{\tru\v} = 1$. Putting this into \eqref{eq:proofs:highdimestimation:obsmodel} leads to
\begin{equation}\label{eq:proofs:selection:modobsmodel}
	\tilde{y}_i = \fobs(\sp{\meas_i}{\tru\v}) = \fobs\Big(\tru\scalsig^{-1} \sp[\Big]{\matr{\lat_i \\ \latnoise_i}}{\matr{\dict \tru\fv \\ \dictnoise \tru\fv}}\Big)  = \fobs(\tru\scalsig^{-1} \sp{\data_i}{\tru\fv}), \quad i = 1, \dots,m,
\end{equation}
whereas the true observations are (cf. \eqref{eq:results:dictionary:obsmodelextd})
\begin{equation}
		y_i = \fobs(\sp{\lat_i}{\trusig}) = \fobs\Big(\sp[\Big]{\matr{\lat_i \\ \latnoise_i}}{\matr{\trusig \\ \vnull}}\Big) = \fobs(\sp{\extd{\lat}_i}{\tru{\extd\sig}}), \quad i = 1, \dots, m.
\end{equation}
Recalling \eqref{eq:results:guarantee:noiseparam} and \eqref{eq:results:guarantee:modddrepresentation}, we observe that in fact $\tilde{y}_i = y_i''$, which implies
\begin{equation}
	\advdev = \Big(\tfrac{1}{m} \sum_{i = 1}^m \abs{\tilde{y}_i - y_i}^2\Big)^{1/2}.
\end{equation}
As signal set we simply choose $L := \extd\dict \sset$. By $\vnull \in \sset$, $\tru\scalsig \geq 1$, and $\scalfac \tru\fv \in \sset$, one concludes that
\begin{equation}
\scalfac \tru\v = \scalfac \tru\scalsig^{-1} \matr{\trusig \\ \dictnoise \tru\fv} = \tru\scalsig^{-1} \matr{\scalfac \dict\tru\fv \\ \scalfac \dictnoise \tru\fv} \in \tru\scalsig^{-1} \extd\dict\sset \subset \extd\dict\sset = L.
\end{equation}
Since (see also \cite[Prop.~2.1]{plan2013robust})
\begin{equation}
	C' \cdot \effdim{L - \scalfac\tru\v} \leq \underbrace{4 C'}_{=: C} \cdot \effdim{L} = C \cdot \effdim{\extd\dict \sset} \leq m,
\end{equation}
all conditions of Theorem~\ref{thm:proofs:highdimestimation} are indeed satisfied.

For the remainder of this proof, let us assume that the ``high-probability-event'' of Theorem~\ref{thm:proofs:highdimestimation} has occurred. Then, the minimizer of \eqref{eq:proofs:highdimestimation:estimator} can be expressed in terms of the actual estimator \eqref{eq:results:estimatorsignal:estimator}:
\begin{align}
\solu\v &= \argmin_{\v \in L} \tfrac{1}{m} \sum_{i = 1}^m \loss(\sp{\meas_i}{\v}, y_i) \\
&= \argmin_{\v \in \extd\dict\sset} \tfrac{1}{m} \sum_{i = 1}^m \loss(\sp{\extd{\lat}_i}{\v}, y_i) \\
&= \extd\dict \cdot \argmin_{\fv \in \sset} \tfrac{1}{m} \sum_{i = 1}^m \loss(\underbrace{\sp{\extd{\lat}_i}{\extd\dict\fv}}_{= \sp{\data_i}{\fv}}, y_i) \\
&= \extd\dict \cdot \underbrace{\argmin_{\fv \in \sset} \tfrac{1}{m} \sum_{i = 1}^m \loss(\sp{\data_i}{\fv}, y_i)}_{\eqref{eq:results:estimatorsignal:estimator}} = \extd\dict \solu\fv = \matr{\dict\solu\fv \\ \dictnoise\solu\fv} = \matr{\solu\sig \\ \dictnoise\solu\fv}. \label{eq:proofs:selection:estimatorequiv}
\end{align}
This finally yields the desired error bound (with $\modeldevconst := \sqrt{2}\modeldevconst'$)
\begin{align}
	\lnorm{\solu\sig - \tru\snrscal\trusig} &= \lnorm[auto]{\smallmatr{\solu\sig \\ \vnull} - \tru\snrscal \smallmatr{\trusig \\ \vnull}} \leq \lnorm[auto]{\smallmatr{\solu\sig \\ \dictnoise\solu\fv} - \scalfac \tru\scalsig^{-1} \smallmatr{\trusig \\ \dictnoise \tru\fv}} \\
	&= \lnorm{\solu\v - \scalfac \tru\v} \stackrel{\eqref{eq:proofs:highdimestimation:bound}}{\leq } \modeldevconst' \bigg( \Big(\frac{4 \effdim{\extd\dict \sset}}{m}\Big)^{1/4} + \advdev \bigg).\label{eq:proofs:selection:errorbound}
\end{align}
\end{proof}

\begin{remark}\label{rmk:proofs:selection}
\begin{rmklist}
\item\label{rmk:proofs:selection:goodsnr}
	The above proof even shows a stronger statement than given in Theorem~\ref{thm:results:selection}. Indeed, \eqref{eq:proofs:selection:errorbound} also implies that
	\begin{equation}
		\lnorm{\dictnoise\solu\fv - \scalfac \tru\scalsig^{-1} \dictnoise \tru\fv} \leq \modeldevconst' \bigg( \Big(\frac{4 \effdim{\extd\dict \sset}}{m}\Big)^{1/4} + \advdev \bigg).
	\end{equation}
	Hence, if $\dictnoise \tru\fv \approx \vnull$ (high SNR) and the bound of the right-hand side is sufficiently small, we can conclude that $\dictnoise\solu\fv \approx \vnull$. This basically means that the SNR is still high when using the estimated feature vector $\solu\fv$ instead of $\tru\fv$.
\item\label{rmk:proofs:selection:rescaledobs}
	As already mentioned in Remark~\ref{rmk:results:guarantee:modddrepresentation}, the proof of Theorem~\ref{thm:results:selection} reveals why the outputs of \eqref{eq:results:dictionary:eddrepresentation},
	\begin{equation}
		y_i' = \fobs(\sp{\data_i}{\tru\fv}) = \fobs(\sp{\extd\lat_i}{\tru{\extd\sig}'}), \quad i = 1, \dots, m,
	\end{equation}
	cannot match with the observation model \eqref{eq:proofs:highdimestimation:obsmodel} in general.\footnote{However, there are at least some special cases, e.g., $\fobs = \sign$, where this approach would still work.}
	In fact, this would require $\lnorm{\tru{\extd\sig}'} = 1$, which in turn can never hold if $\dict\tru\fv = \trusig \in \S^{p-1}$ and $\dictnoise \tru\fv \neq \vnull$.
	That is precisely the reason why we need to consider rescaled outputs $y_i'' = \fobs(\tru\scalsig^{-1} \sp{\data_i}{\tru\fv})$ instead.
\item
	The proof strategy of Theorem~\ref{thm:results:selection} might appear a bit counter-intuitive at first sight: The true observations $y_1, \dots, y_m$ from \ref{assump:model:obsmodel} are rather treated as noisy (and dependent) perturbations of the ``artificial'' output rule defined in \eqref{eq:proofs:selection:modobsmodel}.
	This ``reverse'' perspective is due to the fact that there might not exist a feature vector $\fv \in \R^d$ with $\tru{\extd\sig} = (\trusig, \vnull) = \extd\dict \fv$, i.e., there is no exact representation of $\trusig$ in the extended signal domain (see also Remark~\ref{rmk:results:dictionary}\ref{rmk:results:dictionary:extddict}).
	More precisely, if we would have simply chosen $\meas_i := \lat_i$, $\tru\v := \trusig$, and $L = \dict \sset$, the crucial equivalence of the estimators in \eqref{eq:proofs:selection:estimatorequiv} would only hold if $\ker\dictnoise \intersec \sset = \sset$. But this is not the case in general.
	\qedrmkhere
\end{rmklist}
\end{remark}

Theorem~\ref{thm:results:selection} could have been also deduced as a corollary of Theorem~\ref{thm:results:selection:refined}. However, we have decided to give a separate proof in order to avoid technicalities that may detract from the key techniques of this work.
Theorem~\ref{thm:results:selection:refined} can be derived in a very similar way, using Theorem~\ref{thm:proofs:highdimestimation} again:
\begin{proof}[Proof of Theorem \ref{thm:results:selection:refined}]
	First, observe that $\sqrt{\Covmatr}^{-1}\lat_i \distributed \Normdistr{\vnull}{\Idm{p}}$.
	So, let us set up Theorem~\ref{thm:proofs:highdimestimation} with
	\begin{align}
		\meas_i &:= (\sqrt{\Covmatr}^{-1}\lat_i, \latnoise_i) \distributed \Normdistr{\vnull}{\Idm{p+q}}, \quad i = 1, \dots, m, \\
		\tru\v &:= \tru\scalsig^{-1} \matr{\sqrt{\Covmatr}\trusig \\ \dictnoise \tru\fv}, \quad \text{and} \quad L := \extd\dict_{\Covmatr} \sset.
	\end{align}
	Using the assumptions of Theorem~\ref{thm:results:selection:refined}, one easily shows that $\lnorm{\tru\v} = 1$, $\scalfac \tru\v \in L$, $\tilde{y}_i = y_i''$, and
	\begin{equation}
		\solu\v = \argmin_{\v \in L} \tfrac{1}{m} \sum_{i = 1}^m \loss(\sp{\meas_i}{\v}, y_i)
		= \extd\dict_{\Covmatr} \cdot \argmin_{\fv \in \sset} \tfrac{1}{m} \sum_{i = 1}^m \loss(\sp{\data_i}{\fv}, y_i) = \extd\dict_{\Covmatr} \solu\fv = \matr{\sqrt{\Covmatr}\solu\sig \\ \dictnoise\solu\fv}.
	\end{equation}
	Recall that $y_1, \dots, y_m$ are the true observations, whereas $y_1^0, \dots, y_m^0$ obey \eqref{eq:results:refinements:obsmodel}.
	This leads to
	\begin{equation}
		\Big(\tfrac{1}{m} \sum_{i = 1}^m \abs{\tilde{y}_i - y_i}^2\Big)^{1/2} \leq \Big(\tfrac{1}{m} \sum_{i = 1}^m \abs{y_i'' - y_i^0}^2\Big)^{1/2} + \Big(\tfrac{1}{m} \sum_{i = 1}^m \abs{y_i^0 - y_i}^2\Big)^{1/2} \leq \advdev + \tru\advdev,
	\end{equation}
	and the statement of Theorem~\ref{thm:proofs:highdimestimation} finally yields the desired bound
	\begin{equation}
	\lnorm{\sqrt{\Covmatr}(\solu\sig - \tru\snrscal\trusig)} \leq \lnorm{\solu\v - \scalfac \tru\v} \stackrel{\eqref{eq:proofs:highdimestimation:bound}}{\leq } \modeldevconst' \bigg( \Big(\frac{4 \effdim{\extd\dict_{\Covmatr} \sset}}{m}\Big)^{1/4} + \advdev + \tru\advdev \bigg).
	\end{equation}
\end{proof}